\documentclass[preprint,12pt]{elsarticle}

\usepackage{url}
\usepackage{subfigure}
\usepackage{amsmath}
\usepackage{threeparttable}
\usepackage{algorithm}
\usepackage{algorithmic}
\usepackage{cases}
\usepackage[dvipsone]{epsfig}
\usepackage{graphicx}
\usepackage{booktabs}
\usepackage{epstopdf}
\usepackage{color}
\usepackage{amsfonts}
\usepackage{dsfont}
\usepackage{amssymb}
\usepackage{multirow}
\usepackage{amsthm}
\theoremstyle{definition}
\newtheorem{thm}{Theorem}[section]
\newtheorem{lem}{Lemma}[section]

\newtheorem{cor}{Corollary}[section]
\usepackage{algorithm}
\usepackage{algorithmic}
\usepackage{listings}

\begin{document}

\begin{frontmatter}



\title{Insensitive Stochastic Gradient Twin Support Vector Machines for Large Scale Problems}


\author[1]{Zhen Wang}
\address[1]{School of Mathematical Sciences, Inner Mongolia University, Hohhot,
010021, P.R.China}

\author[2]{Yuan-Hai Shao \corref{cor1}}
\cortext[cor1]{Corresponding author. Tel./Fax:(+86)0571-87313551.}\ead{shaoyuanhai21@163.com}
\address[2]{School of Economics and Management, Hainan University, Haikou, 570228, P.R. China}

\author[1]{Lan Bai}

\author[3]{Li-Ming Liu}
\address[3]{School of Statistics, Capital University of Economics and Business, Beijing, 100070, P.R.China}

\author[4]{Nai-Yang Deng}
\address[4]{College of Science China Agricultural University, Beijing, 100083, P.R.China}

\begin{abstract}
Stochastic gradient descent algorithm has been successfully applied on support vector machines (called PEGASOS) for many classification problems. In this paper, stochastic gradient descent algorithm is investigated to twin support vector machines for classification. Compared with PEGASOS, the proposed stochastic gradient twin support vector machines (SGTSVM) is insensitive on stochastic sampling for stochastic gradient descent algorithm. In theory, we prove the convergence of SGTSVM instead of almost sure convergence of PEGASOS. For uniformly sampling, the approximation between SGTSVM and twin support vector machines is also given, while PEGASOS only has an opportunity to obtain an approximation of support vector machines. In addition, the nonlinear SGTSVM is derived directly from its linear case. Experimental results on both artificial datasets and large scale problems show the stable performance of SGTSVM with a fast learning speed.
\end{abstract}


\end{frontmatter}


\section{Introduction}
Support vector machines (SVM), being powerful tool
for classification \cite{SVM1,SVM2,SVM3}, have already
outperformed most other classifiers in a wide variety of applications
\cite{A1,A2,A3}. Different from SVM with a pair of parallel hyperplanes, twin support vector machines (TWSVM) \cite{TWSVM,TBSVM} with a pair of nonparallel
hyperplanes has been proposed and developed, e.g., twin bounded support vector machines (TBSVM) \cite{TBSVM},
twin parametric margin support vector machines (TPMSVM) \cite{TPMSVM}, and weighted Lagrangian twin support vector machines (WLTSVM) \cite{WLTSVM}. These classifiers have been widely applied in many practical problems \cite{TWSVMdescent,STPMSVM,DTWSVM,TWSVC,MLTSVM,T1,T2,T3,T4,T5}. In the training stage, SVM solves a quadratic programming problem (QPP), whereas TWSVM solve two smaller QPPs by traditional solver such as interior method \cite{SVM2,KKT2,TWSVM}. However, neither SVM nor TWSVM based on these solvers can deal with the large scale problem, especially millions of samples.

In order to deal with the large scale problem, many improvements were proposed, e.g., for SVM, sequential minimal optimization, coordinate decent method, trust region Newton, and stochastic gradient descent algorithm (SGD) in \cite{SMO,SVMlight,LIBSVM,LIBLINEAR,PEGASOS}, and for TWSVM, successive overrelaxation technique, Newton-Armijo algorithm, and dual coordinate decent method in \cite{TBSVM,STPMSVM,LNSVM}. The stochastic gradient descent algorithm for SVM (PEGASOS) \cite{SGDSVM1,SGDSVM2,PEGASOS,ASGD} attracts a great attention, because it partitions the large scale problem into a series of subproblems by stochastic sampling with a suitable size. It has been proved that PEGASOS is almost sure convergent, and thus is able to find an approximation of the desired solution with high probability \cite{SGDconverge,SGDSVM2,PEGASOS}.
The existing experiments confirm the effectiveness of these algorithms with an amazing learning speed.

\begin{figure*}
\begin{center}
\subfigure[]
{\includegraphics[width=0.230\textheight]{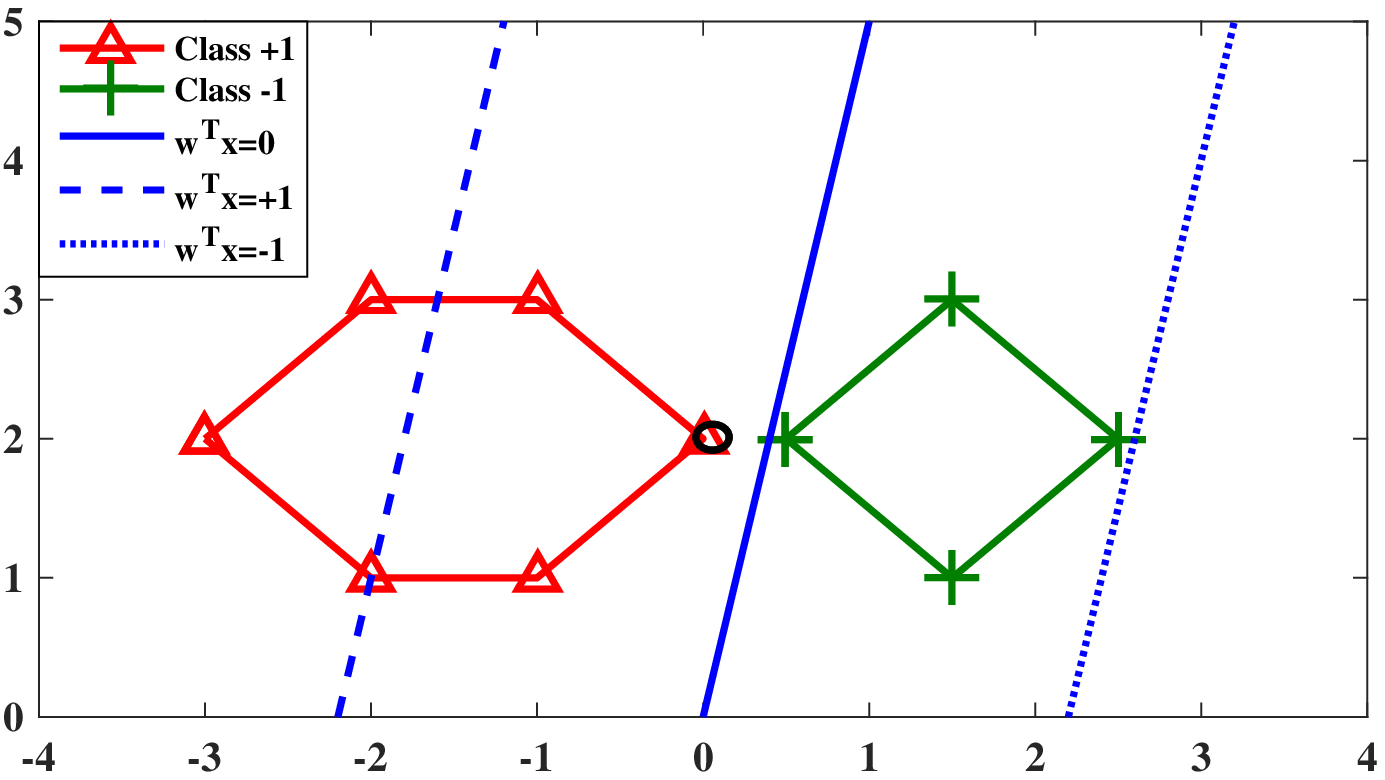}}
\subfigure[]
{\includegraphics[width=0.230\textheight]{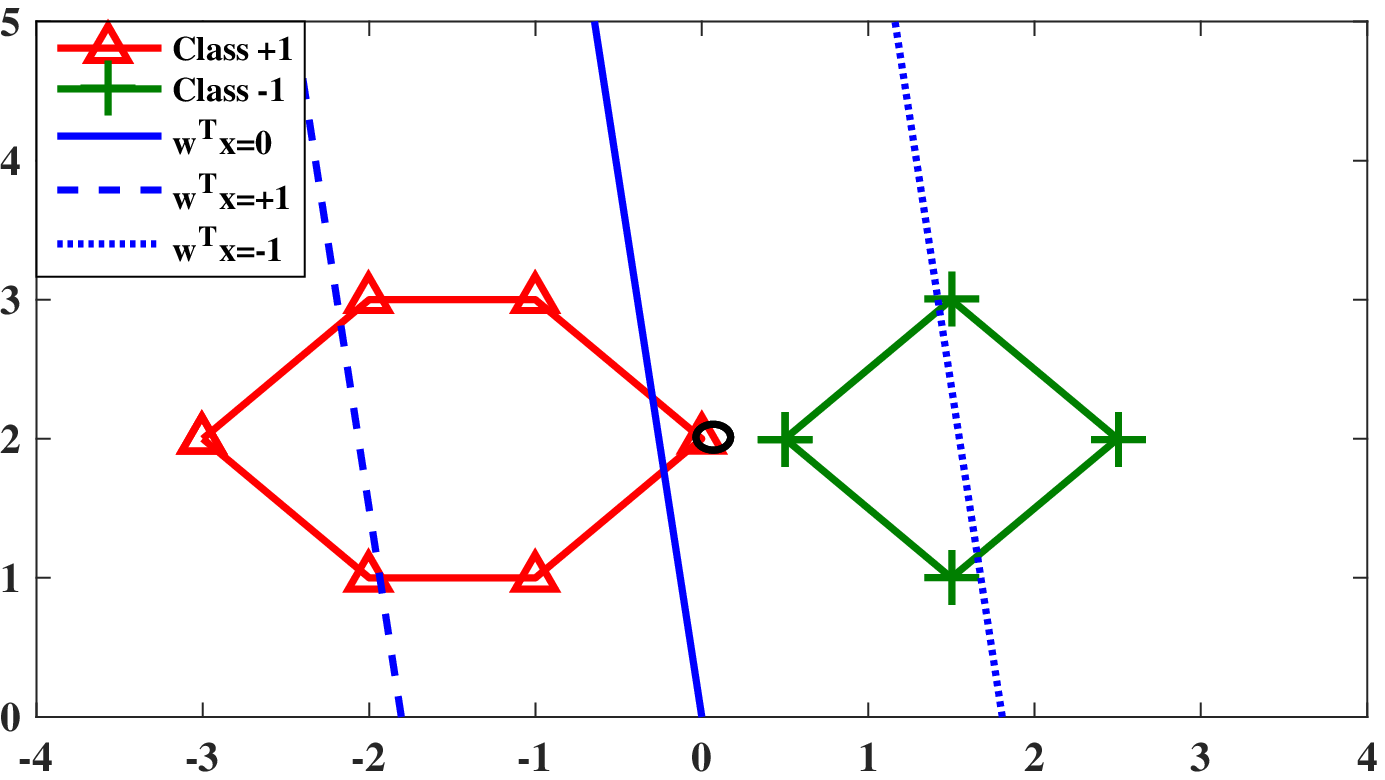}}
\subfigure[]
{\includegraphics[width=0.230\textheight]{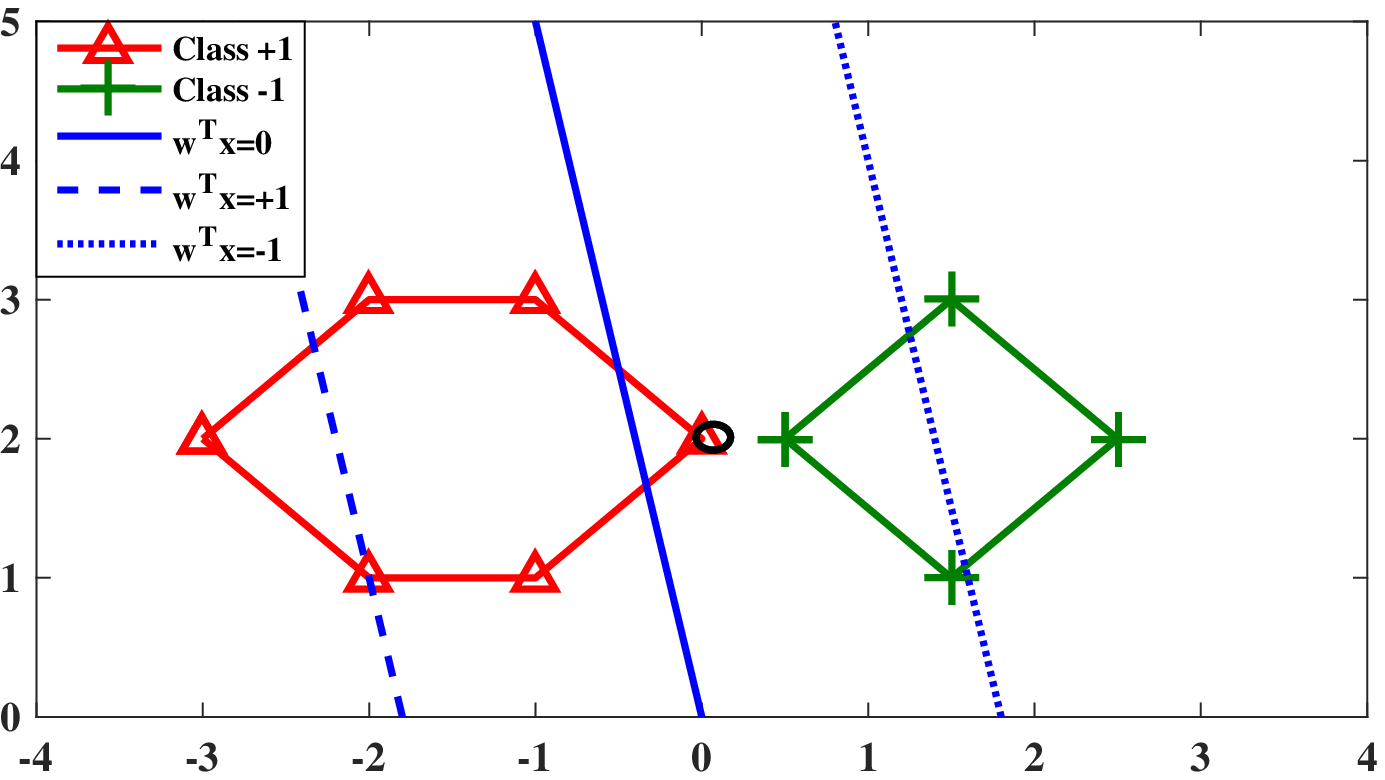}}
\\
\caption{PEGASOS on 10 samples from two classes. (i) Training includes all of the 10 samples with 11 iterations, and the circle sample is used twice; (ii) Training includes all of the 10 samples with 28 iterations, and the circle sample is used once; (iii) Training includes 9 samples with 27 iterations, where the circle sample is excluded.} \label{Exp2svm}
\end{center}
\end{figure*}

\begin{figure*}
\begin{center}
\subfigure[]
{\includegraphics[width=0.230\textheight]{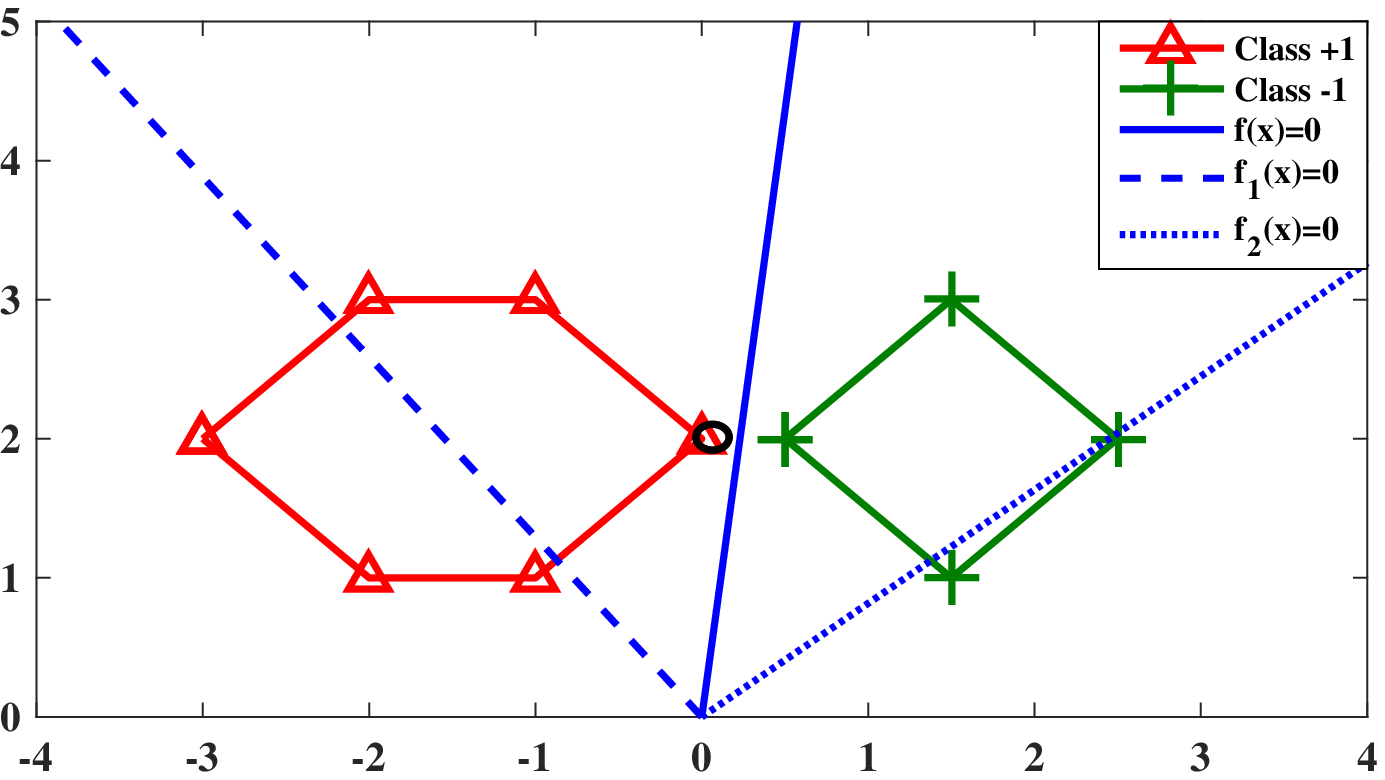}}
\subfigure[]
{\includegraphics[width=0.230\textheight]{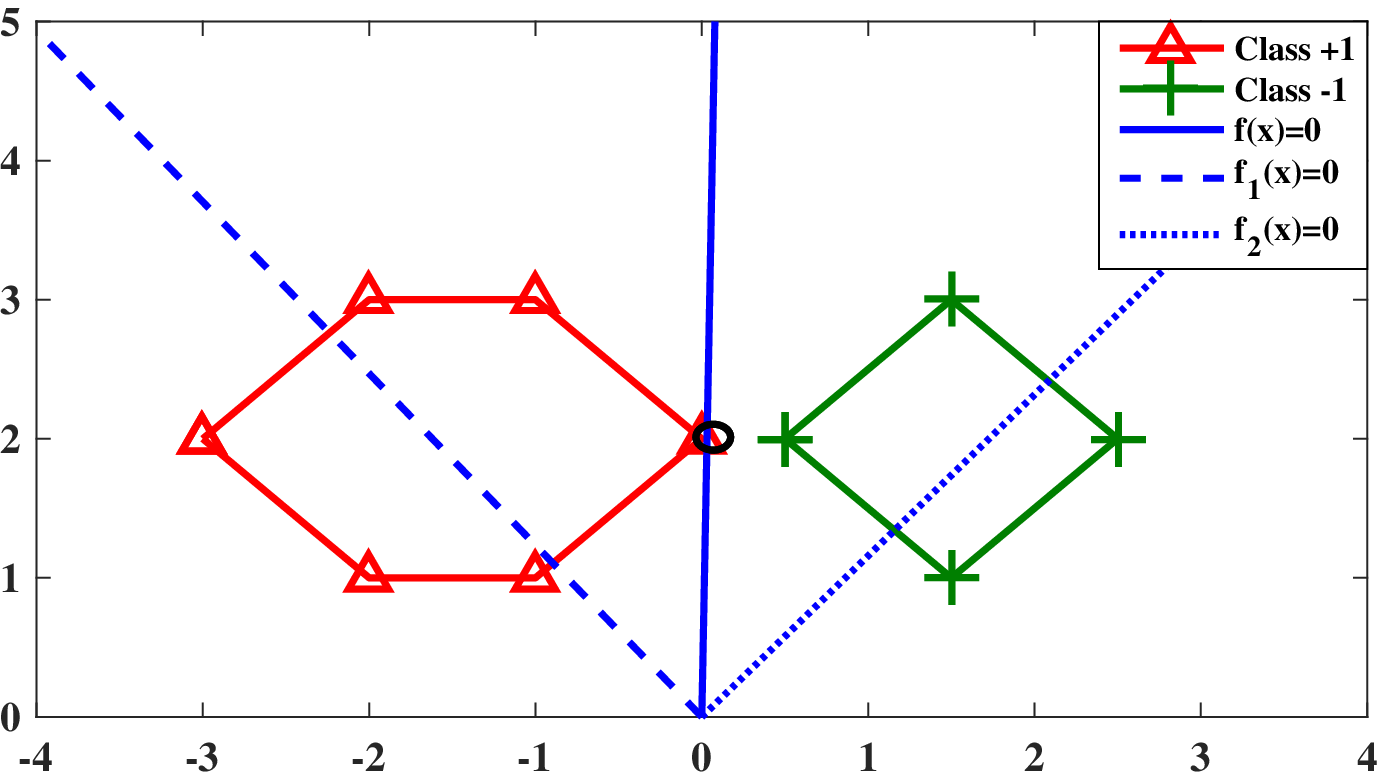}}
\subfigure[]
{\includegraphics[width=0.230\textheight]{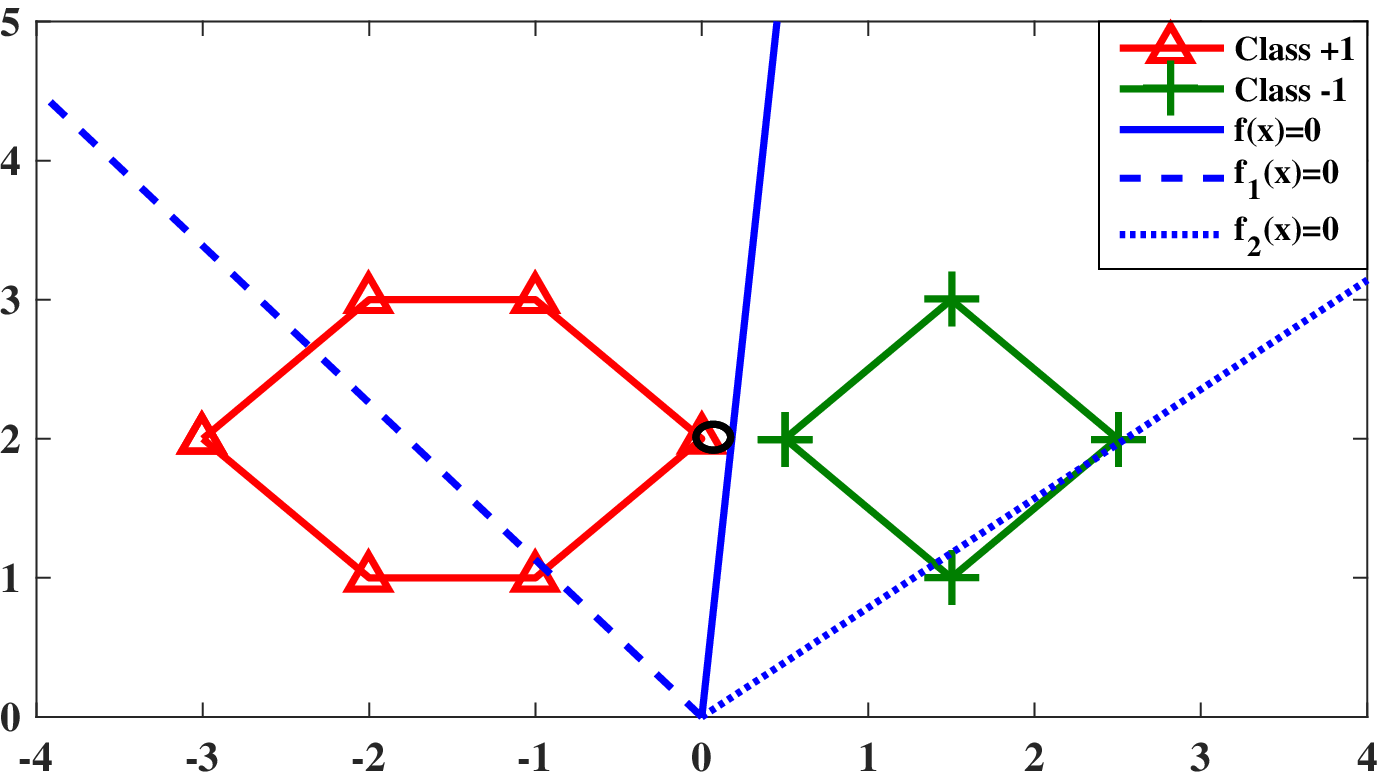}}
\\
\caption{SGTSVM on 10 samples from two classes. (i) Training includes all of the 10 samples with 7 iterations, and the circle sample is used twice; (ii) Training includes all of the 10 samples with 16 iterations, and the circle sample is used once; (iii) Training includes 9 samples with 15 iterations, where the circle sample is excluded.} \label{Exp2twin}
\end{center}
\end{figure*}

However, for large scale problem, the stochastic sampling in SGD may bring some difficulties to SVM due to only a small subset of the dataset is selected for training. In fact, if the subset is not suitable, PEGASOS would be weak. It is well known that in SVM the support vectors (SVs), a small subset of the dataset, decides the final classifier. If the stochastic sampling does not include the SVs sufficiently, the classifier would lose some generalizations. Figure \ref{Exp2svm} is a toy example for PEGASOS. There are two classes in this figure, where the positive and negative classes respectively include 6 and 4 samples, and the circle is one of the potential SVs. The solid blue line is the separating line obtained by PEGASOS with three different sampling: (i) strengthening the circle sample; (ii) infrequently using the circle sample; (iii) ignoring the circle sample. Figure \ref{Exp2svm} shows that the circle sample plays an important role on the separating line, and infrequently using or ignoring this sample would lead to misclassify.

Compared with SVM, it is significant that TWSVM is more stable for sampling and does not strongly depend on some special samples such as the SVs \cite{TWSVM,TBSVM}, which indicates SGD is more suitable for TWSVM. Therefore, in this paper, we propose a stochastic gradient twin support vector machines (SGTSVM). Different from PEGASOS, our method selects two samples from different classes randomly in each iteration to construct a pair of nonparallel hyperplanes. Due to TWSVM fits all of the training samples, our method is stable for the stochastic sampling and thus gains well generalizations. Moreover, the characteristics inherited from TWSVM result in that our SGTSVM suits for many cases, e.g., ``cross planes'' dataset \cite{GEPSVM} and preferential classification \cite{TWSVM}. As the above toy example, Figure \ref{Exp2twin} shows the corresponding results by SGTSVM. Comparing Figure 2 with Figure 1, it is clear that SGTSVM performs better than PEGASOS.

The main contributions of this paper includes:

\noindent
(i) a SGD-based TWSVM (SGTSVM) is proposed, and it is very easy to be extended to other TWSVM-type classifiers;

\noindent
(ii) we prove that the proposed SGTSVM is convergent, instead of almost sure convergence in PEGASOS;

\noindent
(iii) for the uniformly sampling, it is proved that the original objective of the solution to SGTSVM is bounded by the optimum of TWSVM, which indicates the solution to SGTSVM is an approximation of the optimal solution to TWSVM, while PEGASOS only has an opportunity to obtain an approximation of the optimal solution to SVM (more information please see Corollaries 1 and 2 in \cite{PEGASOS});

\noindent
(iv) the nonlinear case of SGTSVM is obtained directly based on its original problem;

\noindent
(v) each iteration of SGTSVM includes no more than $8n+4$ multiplications without additional storage, so it is the fastest one than other proposed TWSVM-type classifiers.


The rest of this paper is organized as follow. Section 2 briefly reviews SVM, PEGASOS, and TWSVM. Our linear and nonlinear SGTSVMs together with the theoretical analysis are elaborated in Section 3. Experiments are arranged in Section 4. Finally, we give the conclusions.

\section{Related Works}
Consider a binary classification problem in the $n$-dimensional real
space $R^n$. The set of training samples is represented by
$X\in R^{n\times m}$, where $x\in R^{n}$ is the sample
with the label $y\in\{+1,-1\}$. We further organize
the $m_1$ samples of Class $+1$ into a matrix $X_1 \in R^{n\times m_1}$
and the $m_2$ samples of Class $-1$ into a matrix $X_2 \in
R^{n\times m_2}$. Below, we give a brief outlines of some related works.

\subsection{SVM}
Support vector machines (SVM) \cite{SVM1,CSVM} searches for a separating hyperplane
\begin{eqnarray}
w^{\top}x+b=0,
\end{eqnarray}
where $w\in R^n$ and $b\in R$. By introducing the
regularization term, the primal problem of SVM can
be expressed as a QPP as follow
\begin{eqnarray}\label{OSVM}
\begin{array}{ll}
 \underset{w,b}{\min} ~~~~ \frac{1}{2}||w||^2+\frac{c}{m}e^\top \xi \\
 \hbox{s.t.\ }~ ~~~~~D(X^\top w+b)\geq e -\xi,~~  \xi\geq0,
\end{array}
\end{eqnarray}
where $||\cdot||$ denotes the $L_2$ norm, $c>0$ is a parameter with some quantitative meanings
\cite{CSVM}, $e$ is a vector of ones with an appropriate dimension, $\xi\in R^m$ is the slack vector,
and $D=\text{diag}(y_1,\ldots,y_m)$. Note that the minimization of the
regularization term $ \|w\|^2$ is equivalent to
maximize the margin between two parallel supporting
hyperplanes $w^{\top}x+b=\pm1$. And the structural risk
minimization principle is implemented in this problem \cite{SVM1}.

\subsection{PEGASOS}
PEGASOS \cite{SGDSVM2,PEGASOS} considers a strongly convex problem by modifying \eqref{OSVM} as follow
\begin{equation}\label{SVM}
\begin{array}{l}
\underset{w}{\min}~~
\frac{1}{2}||w||^2+\frac{c}{m}e^\top\xi\\
s.t.~~~~DX^\top w\geq e-\xi,\xi\geq0,
\end{array}
\end{equation}
and recasts the above problem to
\begin{equation}\label{SVM2}
\begin{array}{l}
\underset{w}{\min}~~
\frac{1}{2}||w||^2+\frac{c}{m}e^\top(e-DX^\top w)_+,
\end{array}
\end{equation}
where $(\cdot)_+$ replaces negative components of a vector by zeros.

In the $t$th iteration ($t\geq1$), PEGASOS constructs a temporary function, which is defined by a random sample $x_t\in X$ as
\begin{equation}\label{SGDSVM}
\begin{array}{l}
g_t(w)=\frac{1}{2}||w||^2+c(1-y_tw^\top x_t)_+.
\end{array}
\end{equation}
Then, starting with an initial $w_1$, PEGASOS iteratively updates $w_{t+1}=w_{t}-\eta_t\nabla_{w_t}g_t(w)$ for $t\geq1$, where $\eta_t=1/t$ is the step size and $\nabla_{w_t}g_t(w)$ is the sub-gradient of $g_t(w)$ at $w_t$,
\begin{equation}
\begin{array}{l}
\nabla_{w_t}g_t(w)=w_t-cy_{t}x_{t} \text{sign}(1-y_{t}w_t^\top x_t)_+.
\end{array}
\end{equation}
When some terminate conditions are satisfied, the last $w_{t}$ is outputted as $w$. And a new sample $x$ can be predicted by
\begin{equation}
\begin{array}{l}
y=\text{sign}(w^\top x).
\end{array}
\end{equation}

It has been proved that the average solution $\bar{w}=\frac{1}{T}\sum\limits_{t=1}^Tw_t$ is bounded by the optimal solution $w^*$ to \eqref{SVM2} with $o(1)$, and thus PEGASOS has with a probability of at least $1/2$ to find a good approximation of $w^*$ \cite{PEGASOS}. The authors of \cite{PEGASOS} also pointed out that $w_{T}$ is often used instead of $\bar{w}$ in practice.
The sample $x_t$ which is selected randomly can be replaced with a small subset belonging to the whole dataset, and  the subset only including a sample is often used in practice \cite{SGDSVM2,PEGASOS,ASGD}. In order to extend the generalization ability of PEGASOS, the bias term $b$ in SVM can be appended to PEGASOS by replacing $g(w_t)$ of \eqref{SGDSVM} with
\begin{equation}
\begin{array}{l}
g(w_t,b)=\frac{1}{2}||w_t||^2+C(1-y_t(w_t^\top x_t+b))_+.
\end{array}
\end{equation}
However, this modification would lead to the function not to be strongly convex and thus yield a slow convergence rate \cite{PEGASOS}.

\subsection{TWSVM}
TWSVM \cite{TWSVM,TBSVM} seeks a pair of nonparallel
hyperplanes in $R^n$ which can be expressed as
\begin{eqnarray}\label{twoplane}
w_1^\top x+b_1=0 ~~\text{and}~~
w_2^\top x+b_2=0,
\end{eqnarray}
such that each hyperplane is close to samples of one class and
has a certain distance from the other class.
To find the pair of nonparallel hyperplanes, it is required to get the solutions to the primal
problems
\begin{eqnarray}\label{10}
\begin{array}{ll}
 \underset{w_1,b_1}{\min} &\frac{1}{2}(||w_1||^2+b_1^2)+\frac{c_1}{2m_1}\|X_1^\top w_1+b_1\|^2+ \frac{c_2}{m_2}e^\top\xi_1\\
 \hbox{s.t.\ } & X_2^\top w_1+b_1-\xi_{1} \leq -e, ~~\xi_{1} \geq 0,
 \end{array}
\end{eqnarray}
and
\begin{eqnarray}
\begin{array}{ll}\label{11}
 \underset{w_2,b_2}{\min}& \frac{1}{2}(||w_2||^2+b_2^2)+\frac{c_3}{2m_2}\|X_2^\top w_{2}+b_{2}\|^{2}+\frac{c_4}{m_1}e^\top\xi_{2}\\
 \hbox{s.t.\ }&  X_1^\top w_{2}+b_{2}+\xi_{2} \geq e, ~~\xi_{2} \geq 0,
 \end{array}
\end{eqnarray}
where $c_{1}$, $c_2$, $c_3$, and $c_{4}$ are positive parameters, $\xi_1\in R^{m_2}$
and $\xi_2\in R^{m_1}$ are slack vectors. Their geometric meaning is
clear. For example, for \eqref{10}, its objective function makes
the samples of Class $+1$ proximal to the hyperplane
$w_{1}^{\top}x+b_{1}=0$ together with the regularization term, while the constraints make each sample
of Class $-1$ has a distance more than $1/||w_1||$ away from the
hyperplane $w_{1}^{\top}x+b_{1}=-1$.

Once the solutions $(w_{1},b_{1})$ and $(w_{2},b_{2})$ to the
problems \eqref{10} and \eqref{11} are respectively obtained, a new point $x\in
R^{n}$ is assigned to which class depends on the distance to
the two hyperplanes in \eqref{twoplane}, i.e.,
\begin{eqnarray}\label{FinalLine}
y=\underset{i}{\arg\min}&
\frac{|w_{i}^{\top}x+b_{i}|}{\|w_{i}\|},
\end{eqnarray}
where $|\cdot|$ is the absolute value.

\section{SGTSVM}
In this section, we elaborate our SGTSVM and give its convergence analysis together with the boundedness.

\subsection{Linear Formation}
Following the notations in Section 2, we recast the QPPs \eqref{10} and \eqref{11} in TWSVM to unconstrained problems
\begin{equation}\label{TWSVM1}
\begin{array}{l}
\underset{w_1,b_1}{\min}~~\frac{1}{2}(||w_1||^2+b_1^2)+
\frac{c_1}{2m_1}||X_1^\top w_1+b_1||^2+\frac{c_2}{m_2}e^\top(e+X_2^\top w_1+b_1)_+,
\end{array}
\end{equation}
and
\begin{equation}\label{TWSVM2}
\begin{array}{l}
\underset{w_2,b_2}{\min}~~\frac{1}{2}(||w_2||^2+b_2^2)+
\frac{c_3}{2m_2}||X_2^\top w_2+b_2||^2
+\frac{c_4}{m_1}e^\top(e-X_1^\top w_2-b_2)_+,
\end{array}
\end{equation}
respectively.

In order to solve the above two problems, we construct a series of strictly convex functions $f_{1,t}(w_1,b_1)$ and $f_{2,t}(w_2,b_2)$ with $t\geq1$ as
\begin{equation}\label{SGDTBSVM1}
\begin{array}{l}
f_{1,t}=\frac{1}{2}(||w_1||^2+b_1^2)+
\frac{c_1}{2}||w_1^\top x_t+b_1||^2
+c_2(1+w_1^\top\hat{x}_t+b_1)_+,
\end{array}
\end{equation}
and
\begin{equation}\label{SGDTBSVM2}
\begin{array}{l}
f_{2,t}=\frac{1}{2}(||w_2||^2+b_2^2)+
\frac{c_3}{2}||w_2^\top\hat{x}_t+b_2||^2
+c_4(1-w_2^\top x_t-b_2)_+,\\
\end{array}
\end{equation}
where $x_t$ and $\hat{x}_t$ are selected randomly from $X_1$ and $X_2$, respectively.

The sub-gradients of the above functions at $(w_{1,t},b_{1,t})$ and $(w_{2,t},b_{2,t})$ can be obtained as
\begin{equation}\label{FormD1}
\begin{array}{l}
\nabla_{w_{1,t}}f_{1,t}=w_{1,t}+c_1(w_{1,t}^\top x_t+b_{1,t})x_{t}
+c_2\hat{x}_{t}\text{sign}(1+w_{1,t}^\top\hat{x}_t+b_{1,t})_+,\\
\nabla_{b_{1,t}}f_{1,t}=b_{1,t}+c_1(w_{1,t}^\top x_t+b_{1,t})
+c_2\text{sign}(1+w_{1,t}^\top\hat{x}_t+b_{1,t})_+,
\end{array}
\end{equation}
and
\begin{equation}\label{FormD2}
\begin{array}{l}
\nabla_{w_{2,t}}f_{2,t}=w_{2,t}+c_3(w_{2,t}^\top \hat{x}_t+b_{2,t})\hat{x}_{t}
-c_4x_{t}\text{sign}(1-w_{2,t}^\top x_t-b_{2,t})_+,\\
\nabla_{b_{2,t}}f_{2,t}=b_{2,t}+c_3(w_{2,t}^\top \hat{x}_t+b_{2,t})
-c_4\text{sign}(1-w_{2,t}^\top x_t-b_{1,t})_+,
\end{array}
\end{equation}
respectively.

Our SGTSVM starts from the initial $(w_{1,1},b_{1,1})$ and $(w_{2,t},b_{2,t})$. Then, for $t\geq1$, the updates are given by
\begin{equation}\label{FormW}
\begin{array}{ll}
w_{1,t+1}=w_{1,t}-\eta_t\nabla_{w_{1,t}}f_{1,t},\\
b_{1,t+1}=b_{1,t}-\eta_t\nabla_{b_{1,t}}f_{1,t},\\
w_{2,t+1}=w_{2,t}-\eta_t\nabla_{w_{2,t}}f_{2,t},\\
b_{2,t+1}=b_{2,t}-\eta_t\nabla_{b_{2,t}}f_{2,t},\\
\end{array}
\end{equation}
where $\eta_t$ is the step size and typically is set to $1/t$. If the terminated condition is satisfied, $(w_{1,t},b_{1,t})$ is assigned to $(w_1,b_1)$, and $(w_{2,t},b_{2,t})$ is assigned to $(w_2,b_2)$. Then, a new sample $x\in R^n$ can be predicted by \eqref{FinalLine}.

The above procedures are summarized in Algorithm \ref{alg:TBSVM}.

\begin{algorithm}[htb]
\caption{SGTSVM Framework.} \label{alg:TBSVM}
\begin{algorithmic}[1]
\REQUIRE ~~\\
Given the training dataset $X_1\in R^{n\times m_1}$ as positive class, $X_2\in R^{n\times m_2}$ as negative class, select parameters $c_1$, $c_2$, $c_3$, $c_4$, and a small tolerance $tol$, typically $tol=1e-3$.

\ENSURE ~~\\
$w_1$, $b_1$, $w_2$, $b_2$.

\STATE set $w_{1,1}$, $b_{1,1}$, $w_{2,1}$, and $b_{2,1}$ be zeros;

For $t=1,2,\ldots$
\STATE Choose a pair of samples $x_t$ and $\hat{x}_t$ from $X_1$ and $X_2$ at random, respectively;

\STATE Compute the $t$th gradients by \eqref{FormD1} and \eqref{FormD2};

\STATE Update $w_{1,t+1}$, $b_{1,t+1}$, $w_{2,t+1}$, and $b_{2,t+1}$ by \eqref{FormW};

\STATE If $||w_{1,t+1}-w_{1,t}||+|b_{1,t+1}-b_{1,t}|<tol$, stop updating $w_{1,t+1}$ and $b_{1,t+1}$, and set $w_1=w_{1,t+1}$, $b_1=b_{1,t+1}$;

\STATE If $||w_{2,t+1}-w_{2,t}||+|b_{2,t+1}-b_{2,t}|<tol$, stop updating $w_{2,t+1}$ and $b_{2,t+1}$, and set $w_2=w_{2,t+1}$, $b_2=b_{2,t+1}$;
\end{algorithmic}
\end{algorithm}

\subsection{Nonlinear Formation}
Now, we extend our SGTSVM to nonlinear case by the kernel trick \cite{GEPSVM,TWSVM,TBSVM,Kernel1,Kernel2,RSVM}. Suppose $K(\cdot,\cdot)$ is the predefined kernel function, then the nonparallel hyperplanes can be expressed as
\begin{eqnarray}
K(x,X)^\top w_{1}+b_{1}=0 ~~\text{and}~~
K(x,X)^\top w_{2}+b_{2}=0.
\end{eqnarray}

The counterparts of \eqref{TWSVM1} and \eqref{TWSVM2} can be formulated as
\begin{equation}
\begin{array}{l}
\underset{w_1,b_1}{\min}~~\frac{1}{2}(||w_1||^2+b_1^2)+
\frac{c_1}{2m_1}||K(X_1,X)^\top w_1+b_1||^2
+\frac{c_2}{m_2}e^\top(e+K(X_2,X)^\top w_1+b_1)_+,
\end{array}
\end{equation}
and
\begin{equation}
\begin{array}{l}
\underset{w_2,b_2}{\min}~~\frac{1}{2}(||w_2||^2+b_2^2)+
\frac{c_3}{2m_2}||K(X_2,X)^\top w_2+b_2||^2
+\frac{c_4}{m_1}e^\top(e-K(X_1,X)^\top w_2-b_2)_+.
\end{array}
\end{equation}

Then, we construct a series of functions with $t\geq1$ as
\begin{equation}\label{NSGDTBSVM1}
\begin{array}{l}
h_{1,t}=\frac{1}{2}(||w_1||^2+b_1^2)+
\frac{c_1}{2}||K(x_t,X)^\top w_1+b_1||^2
+c_2(1+K(\hat{x}_t,X)^\top w_1+b_1)_+,
\end{array}
\end{equation}
and
\begin{equation}\label{NSGDTBSVM2}
\begin{array}{l}
h_{2,t}=\frac{1}{2}(||w_2||^2+b_2^2)+
\frac{c_3}{2}||K(\hat{x}_t,X)^\top w_2+b_2||^2
+c_4(1-K(x_t,X)^\top w_2-b_2)_+.\\
\end{array}
\end{equation}

Similar to \eqref{FormD1}, \eqref{FormD2}, and \eqref{FormW}, the sub-gradients and updates can be obtained. The details are omitted.

For large scale problem, it is time consuming to calculate the kernel $K(\cdot,X)$. However, the reduced kernel strategy, which has been successfully applied for SVM and TWSVM \cite{RSVM,PPSVC,STPMSVM}, can also be applied for our SGTSVM. The reduced kernel strategy replaces $K(\cdot,X)$ with $K(\cdot,\tilde{X})$, where $\tilde{X}$ is a random sampled subset of $X$. In practice, $\tilde{X}$ just needs $0.01\%\sim1\%$ samples from $X$ to get a well performance, reducing the learning time without loss of generalization \cite{PPSVC}.

\subsection{Analysis}
In this subsection, we discuss two issues: (i) the convergence of the solution in SGTSVM; (ii) the relation between the solution in SGTSVM and the optimal one in TWSVM. For convenience, we just consider the first QPP \eqref{TWSVM1} of linear TWSVM together with the SGD formation of linear SGTSVM. The conclusions on another QPP \eqref{TWSVM2} and the nonlinear formations can be obtained easily as the first one.

Let $u=(w^\top,b)^\top$, $Z_1=(X_1^\top,e)^\top$, $Z_2=(X_2^\top,e)^\top$, $z=(x^\top,1)^\top$, and the notations with the subscripts in SGTSVM also comply with this definition. Then, the first QPP \eqref{TWSVM1} is reformulated as
\begin{equation}\label{Origin}
\begin{array}{l}
\underset{u}{\min}~~f(u)=\frac{1}{2}||u||^2+
\frac{c_1}{2m_1}||Z_1u||^2+\frac{c_2}{m_2}e^\top(e+Z_2u)_+.
\end{array}
\end{equation}
Next, we reformulate the $t$th ($t\geq1$) function in SGTSVM as
\begin{equation}\label{Ft}
\begin{array}{l}
f_t(u)=\frac{1}{2}||u||^2+
\frac{c_1}{2}||u^\top z_t||^2+c_2(1+u^\top\hat{z}_t)_+,
\end{array}
\end{equation}
where $z_t$ and $\hat{z}_t$ are the samples selected randomly from $Z_1$ and $Z_2$ for the $t$th iteration, respectively.
The sub-gradient of $f_t(u)$ at $u_t$ is denoted as
\begin{equation}\label{Dt}
\begin{array}{l}
\nabla_t=u_t+
c_1(u^\top z_t)z_t+c_2\hat{z}_t\text{sign}(1+u^\top\hat{z}_t)_+.
\end{array}
\end{equation}

Given $u_1$ and the step size $\eta_t=1/t$, $u_{t+1}$ with $t\geq1$ is updated by
\begin{equation}
\begin{array}{l}
u_{t+1}=u_t-\eta_t\nabla_t,
\end{array}
\end{equation}
i.e.,
\begin{equation}\label{Wt}
\begin{array}{l}
u_{t+1}=(1-\frac{1}{t})u_t-\frac{c_1}{t}z_tz_t^\top u_t-\frac{c_2}{t}\hat{z}_t\text{sign}(1+u_t^\top\hat{z}_t)_+.
\end{array}
\end{equation}

\begin{lem}
For all $t\geq1$, $||\nabla_t||$ and $||u_t||$ have the upper bounds.
\end{lem}

\begin{proof}
The formation \eqref{Wt} can be rewritten as
\begin{equation}\label{3-6}
\begin{array}{l}
u_{t+1}=A_tu_t+\frac{1}{t}v_t,
\end{array}
\end{equation}
where $A_t=\frac{1}{t}((t-1)I-c_1z_tz_t^\top)$, $I$ is the identity matrix, and $v_t=-c_2\hat{z}_t\text{sign}(1+u_t^\top\hat{z}_t)_+$. Note that for sufficient $t$, there is a positive integer $N$ such that for $t>N$, $A_t$ is positive definite, and the largest eigenvalue $\lambda_t$ of $A_t$ is smaller than or equal to $\frac{t-1}{t}$. Based on \eqref{3-6}, we have
\begin{equation}\label{WV}
\begin{array}{l}
u_{t+1}=\prod\limits_{i=N+1}^tA_{t+N+1-i}u_{N+1}+\sum\limits_{i=N+1}^t\frac{1}{i}(\prod\limits_{j=i+1}^tA_{t+i+1-j})v_i.
\end{array}
\end{equation}
For $i\geq N+1$, $||A_{t+N+1-i}u_{N+1}||\leq\lambda_i||u_{N+1}||\leq\frac{i-1}{i}||u_{N+1}||$ \cite{Matrix}. Therefore,
\begin{equation}\label{Wbound}
\begin{array}{l}
||\prod\limits_{i=N+1}^tA_{t+N+1-i}u_{N+1}||\leq\frac{N}{t}||u_{N+1}||,
\end{array}
\end{equation}
and
\begin{equation}\label{Vbound}
\begin{array}{l}
||\frac{1}{i}(\prod\limits_{j=i+1}^tA_{t+i+1-j})v_i||\leq\frac{1}{t}\max\limits_{i\leq t}||v_i||.
\end{array}
\end{equation}
Thus, we have
\begin{equation}\label{WTbound}
\begin{array}{ll}
||u_{t+1}||&\leq\frac{N}{t}||u_{N+1}||+\frac{t-N}{t}\max\limits_{i\leq t}||v_i||\\
&\leq||u_{N+1}||+c_2\max\limits_{z\in Z_2}||z||.
\end{array}
\end{equation}
Let $M$ be the largest norm of the samples in the dataset and \begin{equation}
\begin{array}{l}
G_1=\max\{\max\{||u_1||,\ldots,||u_N||\},||u_{N+1}||+c_2M\}.
\end{array}
\end{equation}
This leads to that $G_1$ is an upper bound of $||u_t||$, and $G_2=G_1+c_1G_1M^2+c_2M$ is an upper bound of $||\nabla_t||$, for $t\geq1$.
\end{proof}

\begin{thm}
The iterative formation \eqref{Wt} of our SGTSVM is convergent.
\end{thm}

\begin{proof}
On the one hand, from \eqref{Wbound} in the proof of Lemma 3.1, we have
\begin{equation}\label{Wlimit}
\begin{array}{l}
\lim\limits_{t\rightarrow\infty}||\prod\limits_{i=N+1}^tA_{t+N+1-i}u_{N+1}||=0,
\end{array}
\end{equation}
which indicates
\begin{equation}\label{W0}
\begin{array}{l}
\lim\limits_{t\rightarrow\infty}\prod\limits_{i=N+1}^tA_{t+N+1-i}u_{N+1}=0.
\end{array}
\end{equation}

On the other hand, from \eqref{Vbound}, we have
\begin{equation}
\begin{array}{l}
\sum\limits_{i=N+1}^t||\frac{1}{i}(\prod\limits_{j=i+1}^tA_{t+i+1-j})v_i||\leq M,
\end{array}
\end{equation}
which indicates that the following limit exists
\begin{equation}\label{Vlimit}
\begin{array}{l}
\lim\limits_{t\rightarrow\infty}\sum\limits_{i=N+1}^t||\frac{1}{i}(\prod\limits_{j=i+1}^tA_{t+i+1-j})v_i||<\infty.
\end{array}
\end{equation}
Note that an infinite series of vectors is convergent if its norm series is convergent \cite{Analysis}. Therefore, the following limit exists
\begin{equation}\label{V0}
\begin{array}{l}
\lim\limits_{t\rightarrow\infty}\sum\limits_{i=N+1}^t\frac{1}{i}(\prod\limits_{j=i+1}^tA_{t+i+1-j})v_i<\infty.
\end{array}
\end{equation}

Combine \eqref{W0} with \eqref{V0}, we conclude that the series $w_{t+1}$ is convergent for $t\rightarrow\infty$.
\end{proof}

Based on the above theorem, it is reasonable to take the terminate condition to be $||u_{t+1}-u_{t}||<tol$. Moreover, if we reform \eqref{WV} by $u_1$, then
\begin{equation}\label{WV0}
\begin{array}{l}
u_{t+1}=\prod\limits_{i=1}^tA_{t+1-i}u_{1}+\sum\limits_{i=1}^t\frac{1}{i}(\prod\limits_{j=i+1}^tA_{t+i+1-j})v_i.
\end{array}
\end{equation}
In order to keep $u_{t+1}$ to be convergent fast, it is suggested to set $u_1=0$.

In the following, we analyse the relation between the solution $u_{t}$ in SGTSVM and the optimal solution $u^*=(w^{*\top},b^*)^\top$ in TWSVM.

\begin{lem}
Let $f_1,\ldots,f_T$ be a sequence of convex functions, and $u_1,\ldots,u_{T+1}\in R^n$ be a sequence of vectors. For $t\geq1$, $u_{t+1}=u_t-\eta_t\nabla_t$, where $\nabla_t$ belongs to the sub-gradient set of $f_t$ at $u_t$ and $\eta_t=1/t$. Suppose $||u_t||$ and $||\nabla_t||$ have the upper bounds $G_1$ and $G_2$, respectively. Then, for all $\theta\in R^n$, we have

\noindent
(i) $\frac{1}{T}\sum\limits_{t=1}^Tf_t(u_t)\leq\frac{1}{T}\sum\limits_{t=1}^Tf_t(\theta)+G_2(G_1+||\theta||)+\frac{1}{2T}G_2^2(1+\ln T)$;

\noindent
(ii) for sufficiently large $T$, given any $\varepsilon>0$, then $\frac{1}{T}\sum\limits_{t=1}^Tf_t(u_t)\leq\frac{1}{T}\sum\limits_{t=1}^Tf_t(\theta)+\varepsilon$.
\end{lem}

\begin{proof}
Since $f_t$ is convex and $\nabla_t$ is the sub-gradient of $f_t$ at $u_t$, we have that
\begin{equation}\label{3-1}
\begin{array}{l}
f_t(u_t)-f_t(\theta)\leq (u_t-\theta)^\top\nabla_t.
\end{array}
\end{equation}

Note that
\begin{equation}\label{3-2}
\begin{array}{l}
(u_t-\theta)^\top\nabla_t=\frac{1}{2\eta_t}(||u_t-\theta||^2-||u_{t+1}-\theta||^2)+\frac{\eta_t}{2}||\nabla_t||^2.
\end{array}
\end{equation}

Combine \eqref{3-1} and \eqref{3-2}, we have
\begin{equation}\label{3-3}
\begin{array}{ll}
&\sum\limits_{t=1}^T(f_t(u_t)-f_t(\theta))\\
\leq&\frac{1}{2}\sum\limits_{t=1}^T\frac{1}{\eta_t}(||u_t-\theta||^2-||u_{t+1}-\theta||^2)+\frac{1}{2}\sum\limits_{t=1}^T(\eta_t||\nabla_t||^2)\\
=&\frac{1}{2}(\sum_{t=1}^T ||u_t-\theta||^2-T||u_{T+1}-\theta||^2)+\frac{1}{2}\sum_{t=1}^T(\eta_t||\nabla_t||^2)\\
\leq&(G_1+||\theta||)\sum\limits_{t=1}^T||u_{T+1}-u_t||+\frac{1}{2}G_2^2(1+\ln T)\\
=&(G_1+||\theta||)\sum\limits_{t=1}^T||\sum\limits_{i=t}^T\frac{1}{i}\nabla_i||+\frac{1}{2}G_2^2(1+\ln T)\\
\leq&TG_2(G_1+||\theta||)+\frac{1}{2}G_2^2(1+\ln T).
\end{array}
\end{equation}
Multiplying \eqref{3-3} by $1/T$ leads to the conclusion (i).


On the other hand, suppose $\lim\limits_{T\rightarrow \infty}u_T=\tilde{u}$, we have $\lim\limits_{T\rightarrow \infty}||u_T||=||\tilde{u}||$. Then, $\lim\limits_{T\rightarrow \infty}\frac{1}{T}\sum\limits_{t=1}^T||u_t-\theta||=\lim\limits_{T\rightarrow \infty}||u_T-\theta||=||\tilde{u}-\theta||$. Note that $\lim\limits_{T\rightarrow \infty} \frac{G_2^2(1+lnT)}{T}=0$. Given any $\varepsilon>0$, for sufficiently large $T$,
\begin{equation}\label{3-4}
\begin{array}{ll}
&\frac{1}{T}\sum\limits_{t=1}^T(f_t(u_t)-f_t(\theta))\\
\leq&\frac{1}{2}(\frac{1}{T}\sum\limits_{t=1}^T||u_t-\theta||^2-||u_{T+1}-\theta||^2)+\frac{1}{2T}G_2^2(1+lnT)\\
\leq&\frac{1}{2}\varepsilon+\frac{1}{2}\varepsilon=\varepsilon.
\end{array}
\end{equation}

\end{proof}

We are now ready to bound the average instantaneous objective \eqref{Ft}.
\begin{thm}
For $f_t$ ($t=1,\ldots,T$) defined as \eqref{Ft} in SGTSVM, $u_{t}$ ($t=1,\ldots,T$) is constructed by \eqref{Wt}, and $u^*$ is the optimal solution to \eqref{Origin}. Then,

\noindent
(i) there are two constants $G_1$ and $G_2$ (actually, they are the upper bounds of $||w_t||$ and $||\nabla_t||$, respectively) such that $\frac{1}{T}\sum\limits_{t=1}^Tf_t(u_t)\leq\frac{1}{T}\sum\limits_{t=1}^Tf_t(u^*)+G_2(G_1+||u^*||)+\frac{1}{2T}G_2^2(1+\ln T)$;

\noindent
(ii) for sufficiently large $T$, given any $\varepsilon>0$, then $\frac{1}{T}\sum\limits_{t=1}^Tf_t(u_t)\leq\frac{1}{T}\sum\limits_{t=1}^Tf_t(u^*)+\varepsilon$.
\end{thm}

\begin{proof}
Obviously, $f_t$ ($t=1,\ldots,T$) is convex. Let $G_1$ and $G_2$ respectively be the upper bounds of $||u_t||$ and $||\nabla_t||$, the conclusions come from Lemmas 3.1 and 3.2.
\end{proof}

In the following, let us discuss the relation between the solutions to SGTSVM and TWSVM with the uniform sampling.

\begin{cor}
Assume the conditions stated in Theorem 3.1 and $m_1=m_2$, where $m_1$ and $m_2$ are the sample number of $X_1$ and $X_2$, respectively. Suppose $T=km_1$, where $k>0$ is an integer, and each sample is selected $k$ times at random. Then

\noindent
(i) $f(u_{T})\leq f(u^*)+G_2(G_1+||u^*||+G_2)+\frac{1}{2T}G_1^2(1+\ln T)$;

\noindent
(ii) for sufficiently large $T$, given any $\varepsilon>0$, then $f(u_{T})\leq f(u^*)+G_2^2+\varepsilon$.
\end{cor}

\begin{proof}
First, we prove that for all $i,j=1,2,\ldots,T$,
\begin{equation}\label{Lips}
\begin{array}{l}
|f_t(u_i)-f_t(u_j)|\leq G_2||u_i-u_j||,~~t=1,2,\ldots,T.
\end{array}
\end{equation}

From the formation of $f_t(u)$, we have
\begin{equation}\label{FtWi}
\begin{array}{ll}
|f_t(u_i)-f_t(u_j)|&\leq\frac{1}{2}|||u_i||^2-||u_j||^2|\\
&+\frac{c_1}{2}|(u_i^\top z_t)^2-(u_j^\top z_t)^2|\\
&+c_2|(1+u_i^\top\hat{z}_t)_+-(1+u_j^\top\hat{z}_t)_+|.
\end{array}
\end{equation}
Since $G_1$ is the upper bound of $||u_t||$ ($t\geq1$) and $M$ is the largest norm of the samples in the dataset, the first part, the second part, and the third part on the right hand of \eqref{FtWi} are respectively
\begin{equation}
\begin{array}{l}
\frac{1}{2}|||u_i||^2-||u_j||^2|\leq G_1||u_i-u_j||,
\end{array}
\end{equation}
\begin{equation}
\begin{array}{ll}
&\frac{c_1}{2}|(u_i^\top z_t)^2-(u_j^\top z_t)^2|\\
=&\frac{c_1}{2}|(u_i+u_j)^\top z_t(u_i-u_j)^\top z_t|\\
\leq&c_1G_1M^2||u_i-u_j||,
\end{array}
\end{equation}
and
\begin{equation}
\begin{array}{ll}
&c_2|(1+u_i^\top\hat{z}_t)_+-(1+u_j^\top\hat{z}_t)_+|\\
=&c_2|(u_i-u_j)^\top\hat{z}_t|\\
\leq&c_2M||u_i-u_j||.
\end{array}
\end{equation}
Therefore, there is a constant $G_2=G_1+c_1G_1M^2+c_2M$ satisfying \eqref{Lips}.

From $u_{t+1}=u_t-\frac{1}{t}\nabla_t$, it is easy to obtain
\begin{equation}\label{Weq}
\begin{array}{l}
u_{t+1}=u_1-\sum\limits_{i=1}^t\frac{1}{i}\nabla_t,~~t=1,2,\ldots,T.
\end{array}
\end{equation}
Thus, for $1\leq i<j\leq T$,
\begin{equation}\label{Wiwj}
\begin{array}{l}
||u_i-u_j||=||\sum\limits_{t=i}^{j-1}\frac{1}{t}\nabla_t||\leq\sum\limits_{t=i}^{j-1}\frac{1}{t}G_2.
\end{array}
\end{equation}

Since $T=km_1=km_2$, for all $u\in R^n$, $\frac{1}{T}\sum\limits_{t=1}^Tf_t(u)=f(u)$. Note that $f(u)$ is the objective of TWSVM. Based on \eqref{Lips} and \eqref{Wiwj}, we have
\begin{equation}
\begin{array}{ll}
&f(u_T)-\frac{1}{T}\sum\limits_{t=1}^Tf_t(u_t)\\
=&\frac{1}{T}\sum\limits_{t=1}^T(f_t(u_T)-f_t(u_t))\\
\leq&\frac{1}{T}\sum\limits_{t=1}^TG_2||u_T-u_t||\\
\leq&\frac{G_2^2(T-1)}{T}\\
\leq&G_2^2.
\end{array}
\end{equation}

Using the Theorem 3.1, we have the conclusion immediately.
\end{proof}

If $m_1\neq m_2$, we can modify the sampling rule to obtain the same result as one in Corollary 3.1.

\begin{cor}
Assume the conditions stated in Corollary 3.1, but $m_1\neq m_2$. Suppose $T=kd(m_1,m_2)$, where $k>0$ is an integer and $d$ is the least common multiple of $m_1$ and $m_2$. The sample in $X_1$ is selected $kd/m_1$ times at random, and the one in $X_2$ is $kd/m_2$ times at random. Then

\noindent
(i) $f(u_{T})\leq f(u^*)+G_2(G_1+||u^*||+G_2)+\frac{1}{2T}G_1^2(1+\ln T)$;

\noindent
(ii) for sufficiently large $T$, given any $\varepsilon>0$, then $f(u_{T})\leq f(u^*)+G_2^2+\varepsilon$.
\end{cor}

Note that for all $u\in R^n$, $\frac{1}{T}\sum\limits_{t=1}^Tf_t(u)=f(u)$. The proof of the above corollary is the same as Corollary 3.1.

The above corollaries provide the approximations of $u^*$ by $u_T$. If the sampling rule is not as stated in these corollaries, these upper bounds no longer holds. However, Kakade and Tewari \cite{SGDconvergeP} have shown a way to obtain a similar bounds with high probability.


\section{Experiments}
In the experiments, we compared our SGTSVM
with SVM \cite{SVM1}, PEGASOS \cite{PEGASOS}, and TWSVM \cite{TWSVM,TBSVM} on several artificial and large scale problems. All of the methods were implemented on a PC with an Intel Core Duo processor (3.4 GHz) with 4 GB RAM.

\begin{figure*}
\begin{center}
\subfigure[TWSVM]
{\includegraphics[width=0.35\textheight]{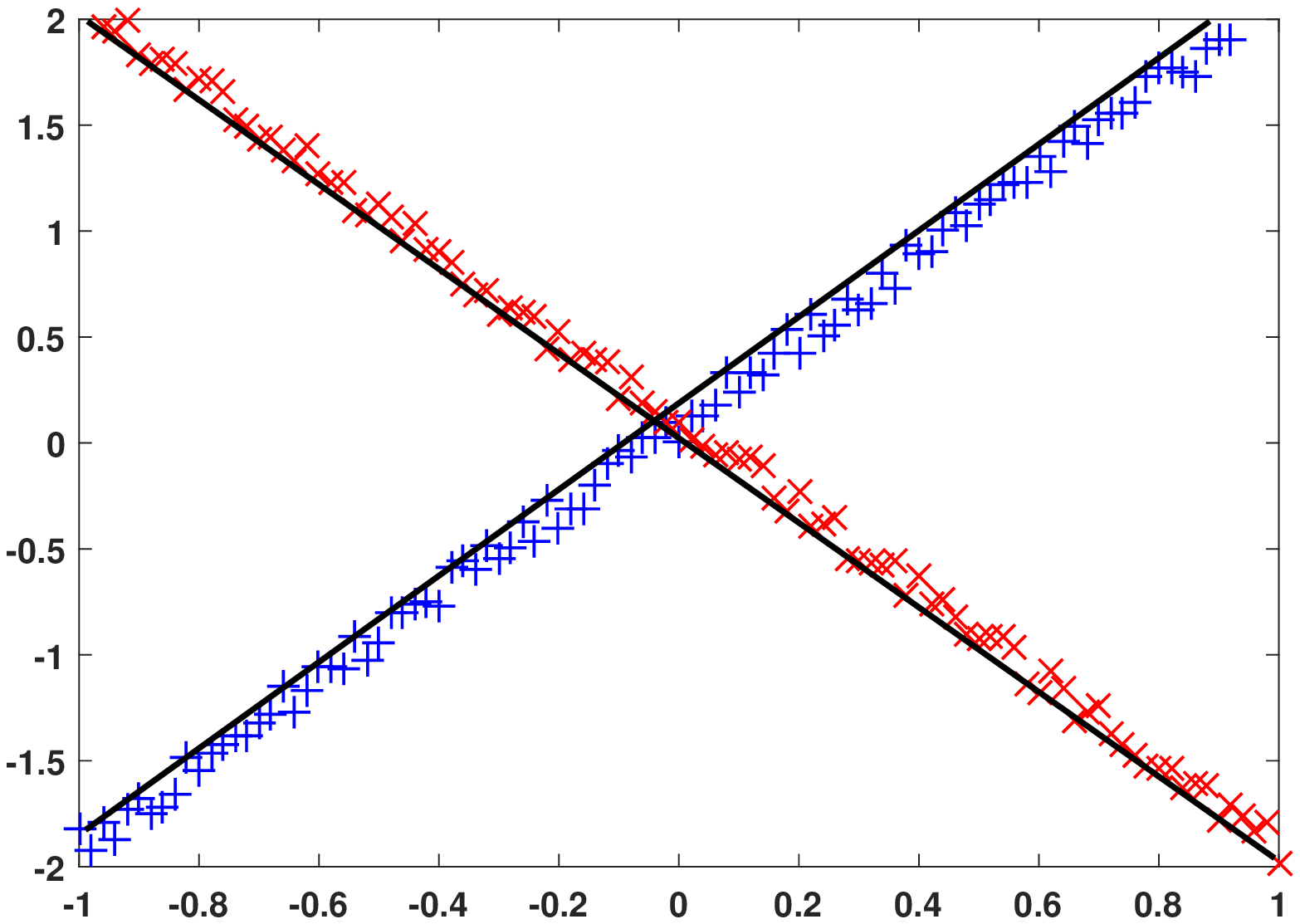}}
\subfigure[SGTSVM]
{\includegraphics[width=0.35\textheight]{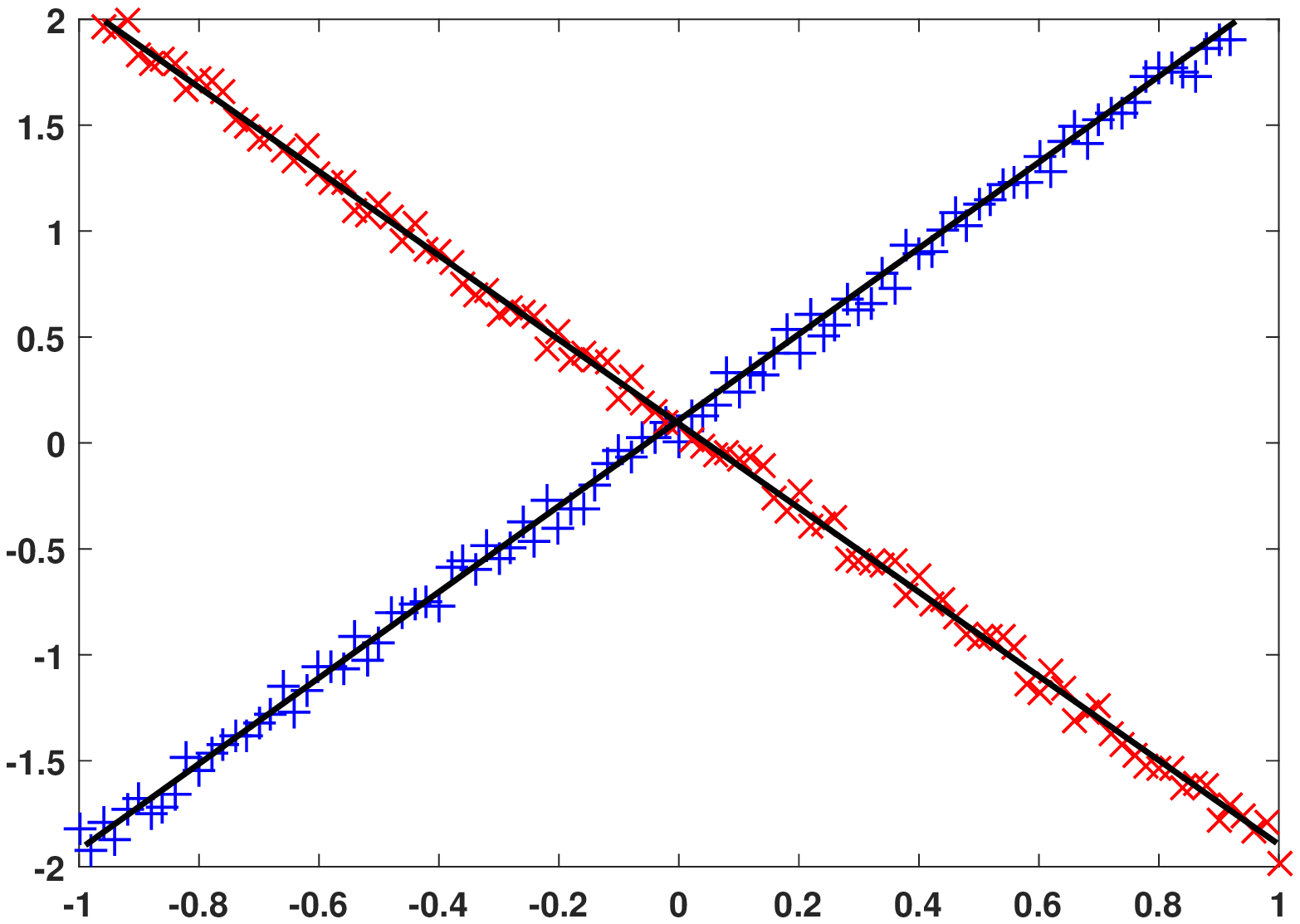}}
\caption{Results of TWSVM and SGTSVM on the ``cross planes'', where the black solid lines are $w_1^\top x+b_1=0$ and $w_2^\top x+b_2=0$, respectively.}\label{CrossPlane}
\end{center}
\end{figure*}

\subsection{Artificial datasets}
On the artificial datasets, PEGASOS, TWSVM, and our SGTSVM were implemented by Matlab \cite{MATLAB}, and the corresponding SGTSVM Matlab codes were uploaded upon
\url{http://www.optimal-group.org/Resource/SGTSVM.html}.

\begin{figure*}[!htb]
\begin{center}
\subfigure[Cross planes]
{\includegraphics[width=0.35\textheight]{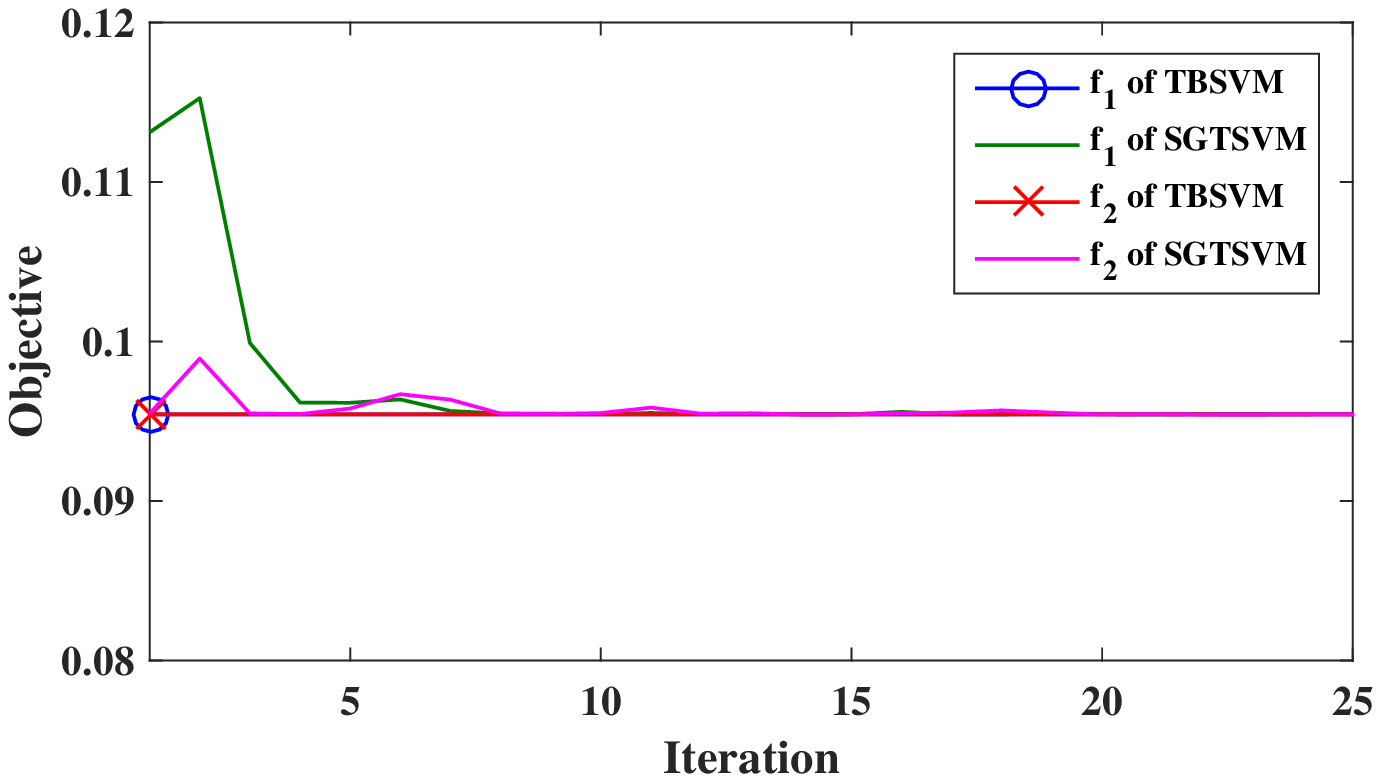}}
\subfigure[Australia]
{\includegraphics[width=0.35\textheight]{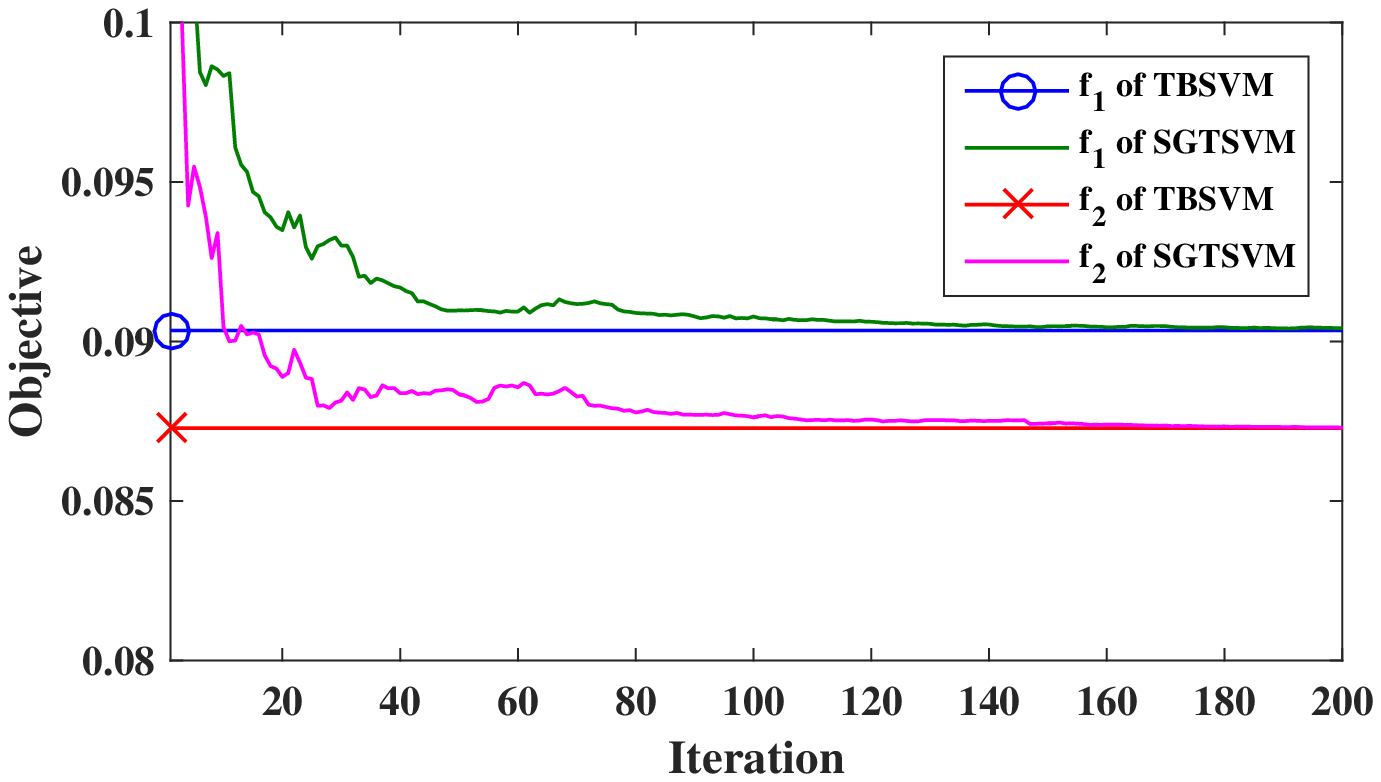}}
\subfigure[Creadit]
{\includegraphics[width=0.35\textheight]{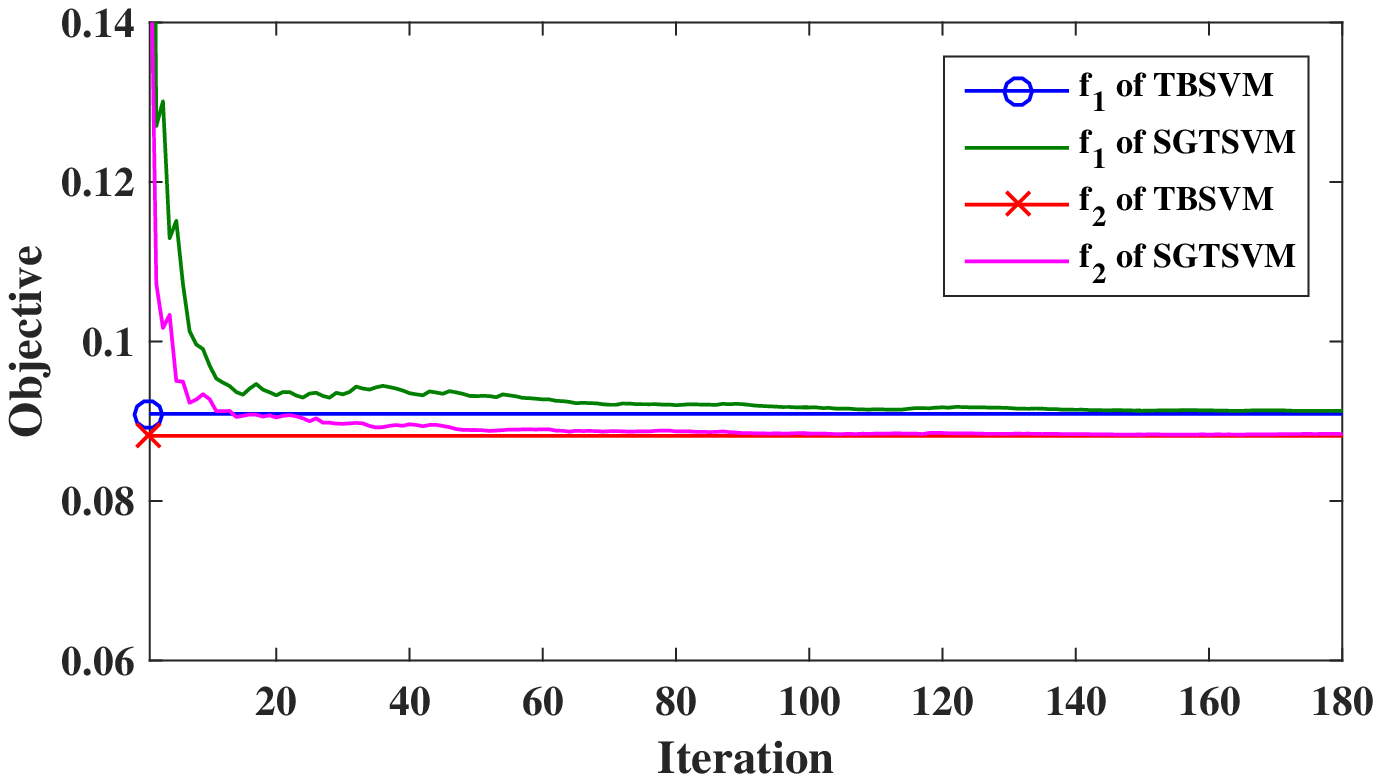}}
\subfigure[Hypothyroid]
{\includegraphics[width=0.35\textheight]{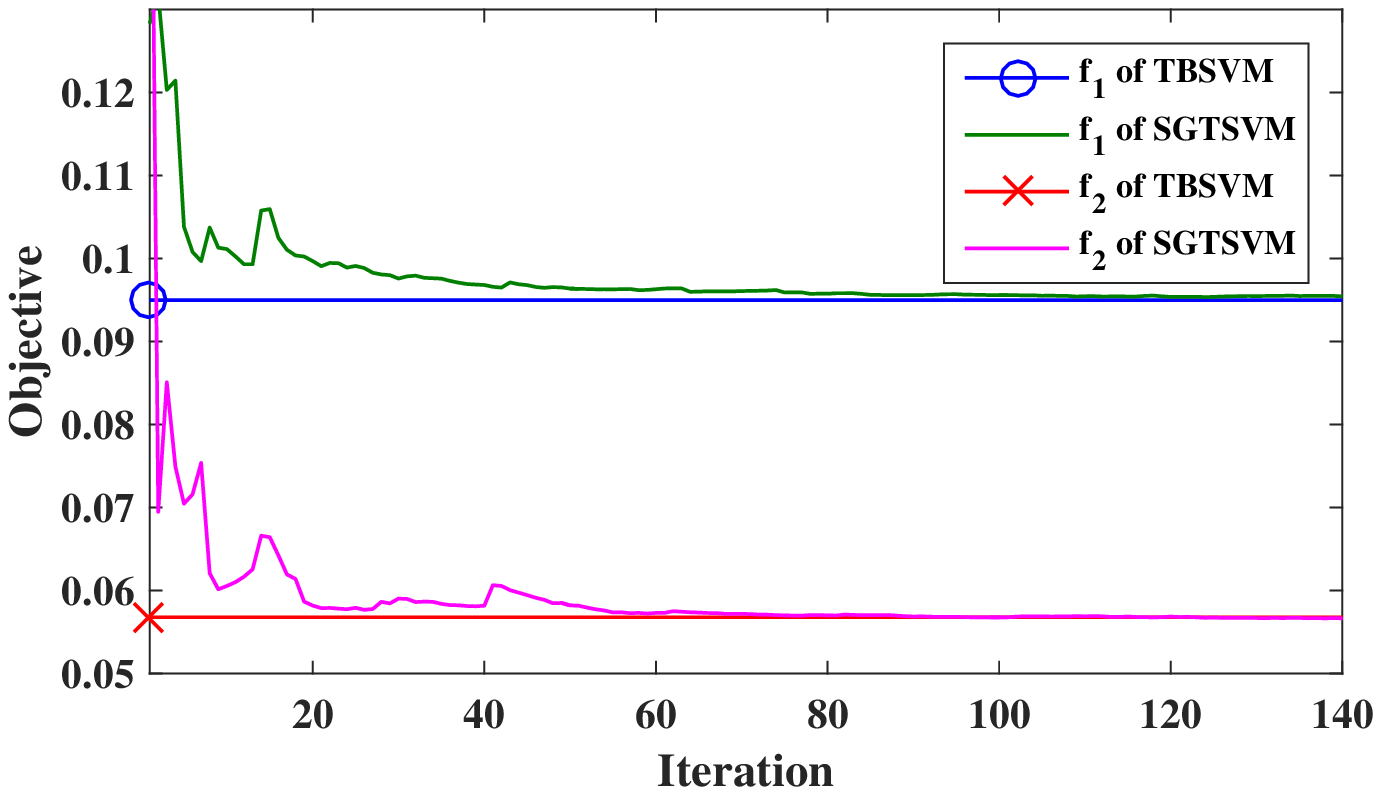}}
\caption{Results of linear TWSVM and SGTSVM on the four datasets, where the vertical axis denotes the objectives of $f_1$ and $f_2$.}\label{Obj}
\end{center}
\end{figure*}

\begin{figure*}[!htb]
\begin{center}
\subfigure[Cross planes]
{\includegraphics[width=0.35\textheight]{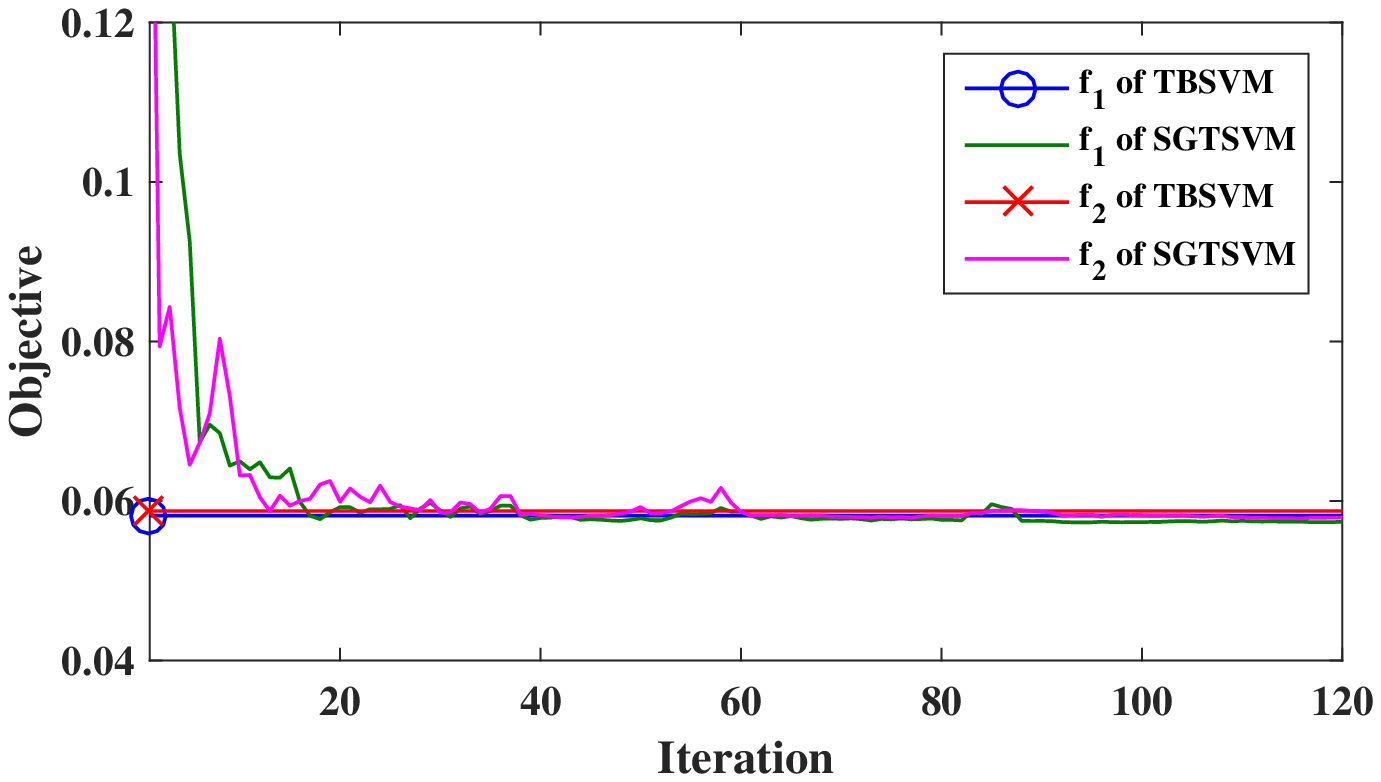}}
\subfigure[Australia]
{\includegraphics[width=0.35\textheight]{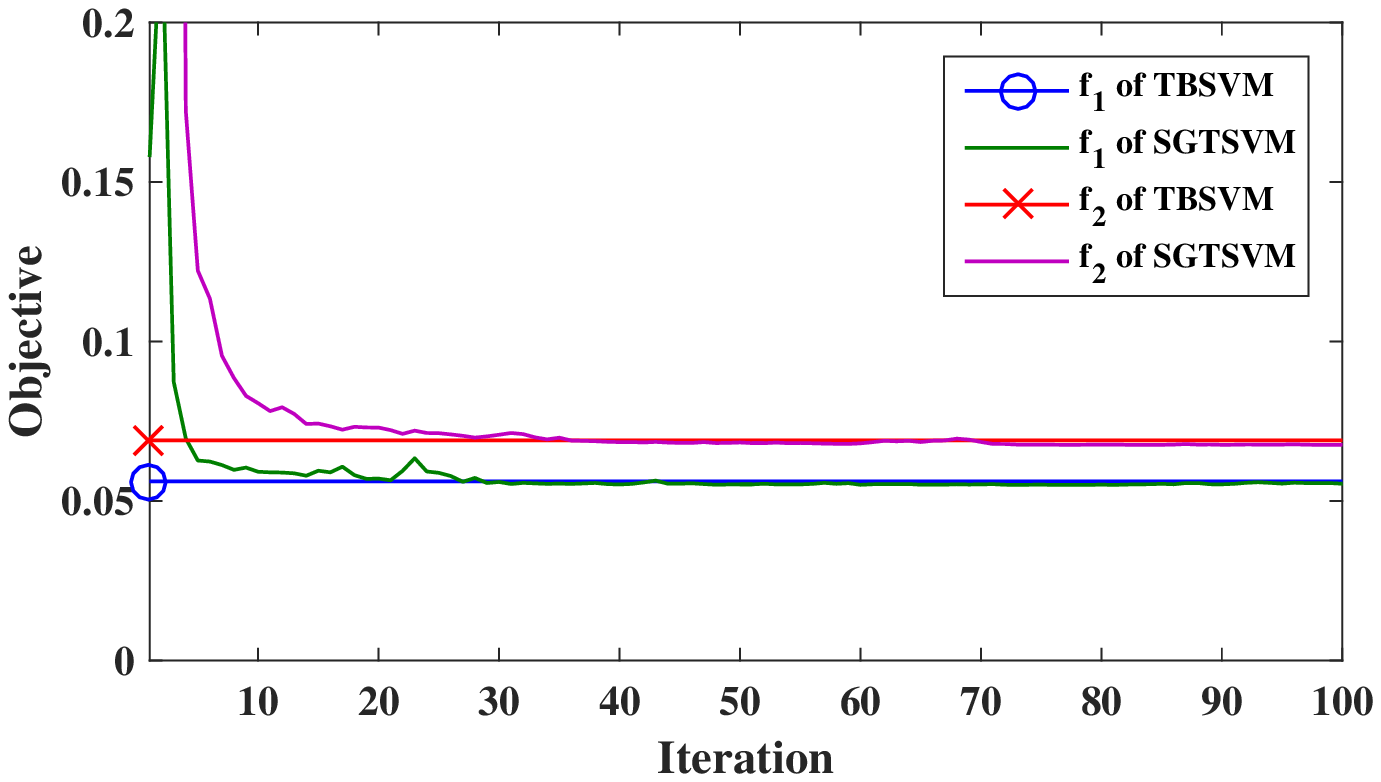}}
\subfigure[Creadit]
{\includegraphics[width=0.35\textheight]{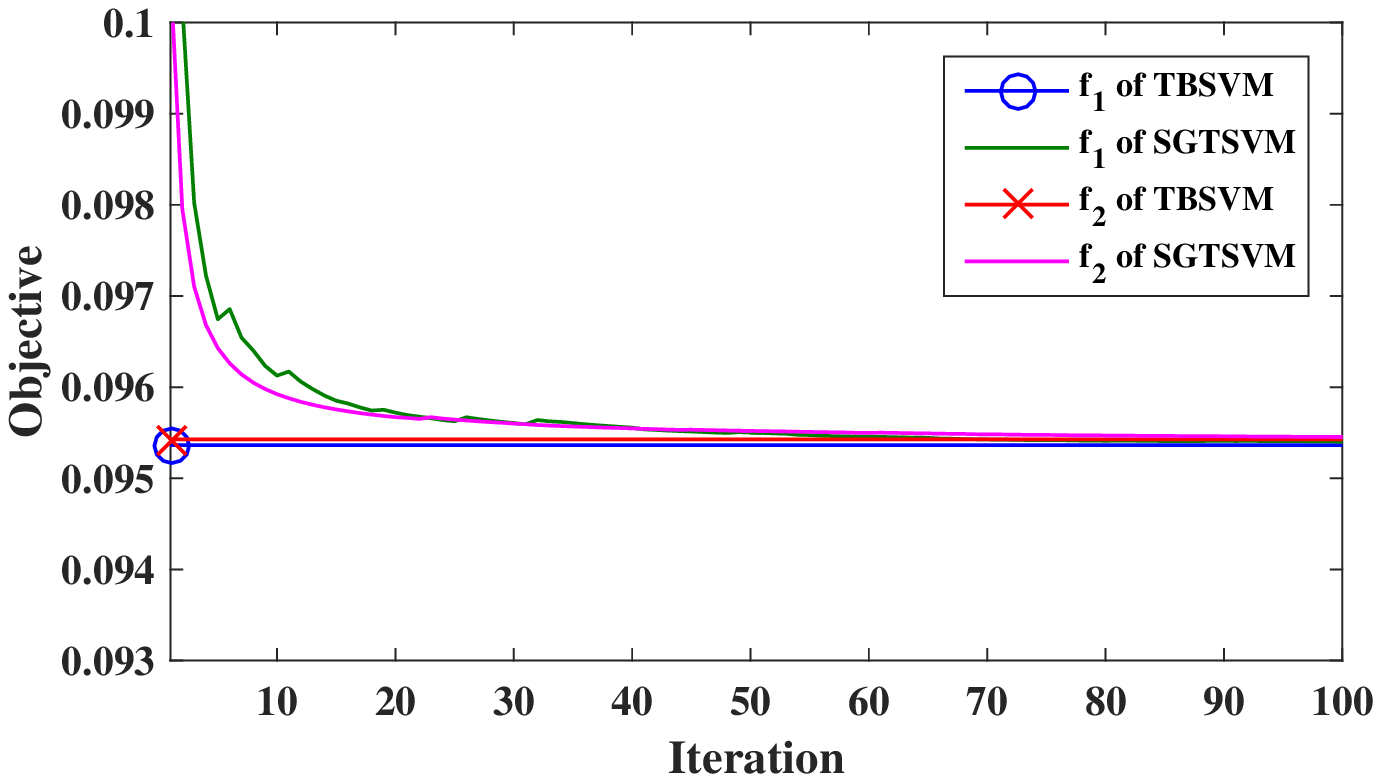}}
\subfigure[Hypothyroid]
{\includegraphics[width=0.35\textheight]{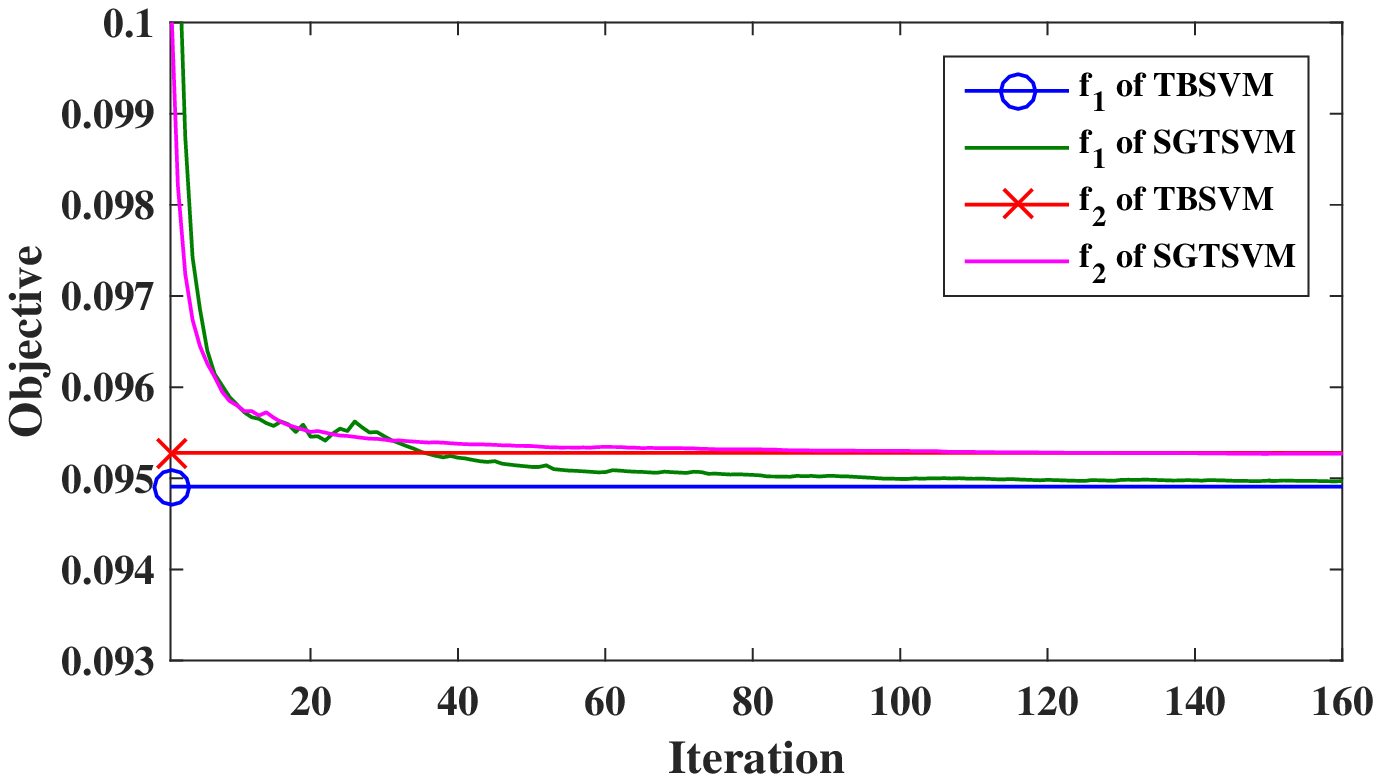}}
\caption{Results of nonlinear TWSVM and SGTSVM on the four datasets, where the vertical is the same as the one in Fig. \ref{Obj}.}\label{nObj}
\end{center}
\end{figure*}

First of all, we consider the similarity between TWSVM and SGTSVM. These two methods were implemented on the ``cross planes'' dataset, where TWSVM was superior on this dataset \cite{TWSVM}. Figure \ref{CrossPlane} shows the proximal lines on the dataset. It is obvious that the two proximal lines by SGTSVM is similar as the ones by TWSVM, so TWSVM and SGTSVM can precisely capture the data distribution, and thus both of them obtain the well classifier.
To measure the similarity quantitatively, the optimums $f_1$ of \eqref{10} and $f_2$ of \eqref{11} in TWSVM were calculated compared with the ones of each iteration in SGTSVM on the ``cross planes'' and some UCI datasets \cite{UCI} (e.g., dataset Australia which includes $690$ samples with $14$ features, dataset Creadit which includes $690$ samples with $15$ features, and dataset Hypothyroid which includes $3,163$ samples with $25$ features). Linear TWSVM, SGTSVM, and their nonlinear versions were implemented, where the Gaussian kernel $K(x,y)=\exp\{-\mu||x-y||^2\}$ was used for nonlinear versions.
The parameters $c_1$, $c_2$, $c_3$, $c_4$, and $\mu$ are fixed to $0.1$. Figure \ref{Obj} shows the results from the two linear classifiers, and Figure \ref{nObj} corresponds to the nonlinear case. In Figures \ref{Obj} and \ref{nObj}, the horizontal axis denotes the iteration of SGTSVM and the vertical axis denotes the objectives $f_1$ and $f_2$ of TWSVM and SGTSVM. Due to the objectives of TWSVM are constant, they are denoted by two horizontal lines, while the objectives of SGTSVM for each iteration are denoted by two broken lines in these figures.
For different datasets, it can be seen that our SGTSVM converges to TWSVM after different iterations. For instance, linear SGTSVM converges to TWSVM after $20$ iterations in Figure \ref{Obj} (a), whereas the same thing appears in Figure \ref{Obj} (b) after $180$ iterations. Generally, SGTSVM converges to TWSVM after $150$ iterations on these datasets either for linear or nonlinear case. Furthermore, the 10-fold cross validation \cite{SVM5} was
used on these datasets. We ran TWSVM and SGTSVM $10$ times, and reported the mean accuracy and standard deviation on Table \ref{Fold10}. The differences of the mean accuracies are no more than $2\%$, which implies the classifiers obtained by TWSVM and SGTSVM do not have significant difference.
\begin{table}
\caption{Mean accuracy (\%) with standard deviation of TWSVM and SGTSVM by 10-fold cross validation.} \centering
\begin{tabular}{lllll}
\hline\noalign{\smallskip}
Data&TWSVM$^\dag$&SGTSVM$^\dag$&TWSVM$^\sharp$&SGTSVM$^\sharp$ \\
\hline
Cross~Planes&96.05$\pm$0.70 &97.71$\pm$0.41&99.01$\pm$2.24      &98.51$\pm$2.15   \\
\hline
Australia&86.87$\pm$0.38&87.34$\pm$0.13&87.10$\pm$0.43&85.21$\pm$0.16\\
\hline
Creadit&85.78$\pm$0.32&85.72$\pm$0.23&86.71$\pm$0.33 &85.21$\pm$0.45\\
\hline
Hypothyroid&98.21$\pm$0.09&97.28$\pm$0.01&98.08$\pm$0.09&98.07$\pm$0.03\\
\hline \noalign{\smallskip}
\end{tabular} \label{Fold10}
$^\dag linear ~case$;$^\sharp nonlinear ~case$.
\end{table}

\begin{figure*}
\begin{center}
\subfigure[PEGASOS]
{\includegraphics[width=0.35\textheight]{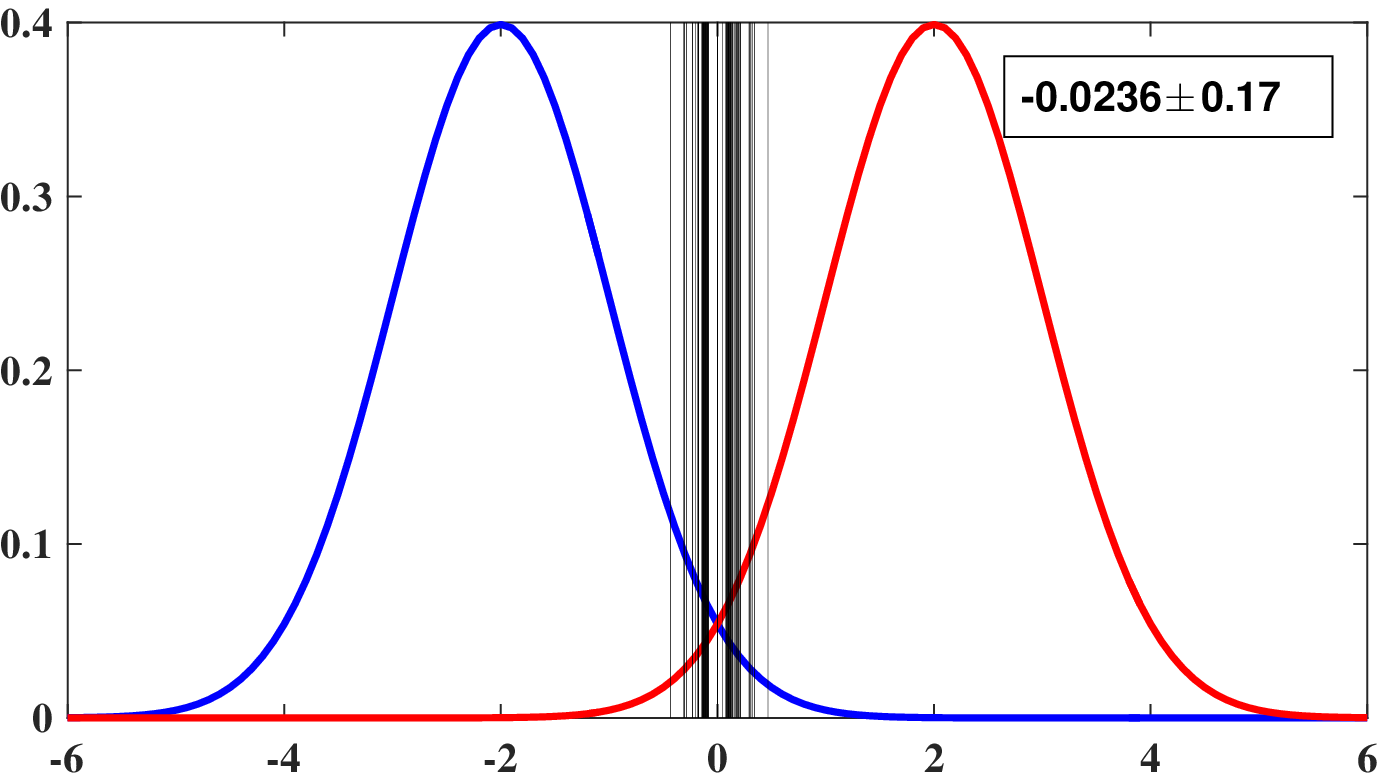}}
\subfigure[SGTSVM]
{\includegraphics[width=0.35\textheight]{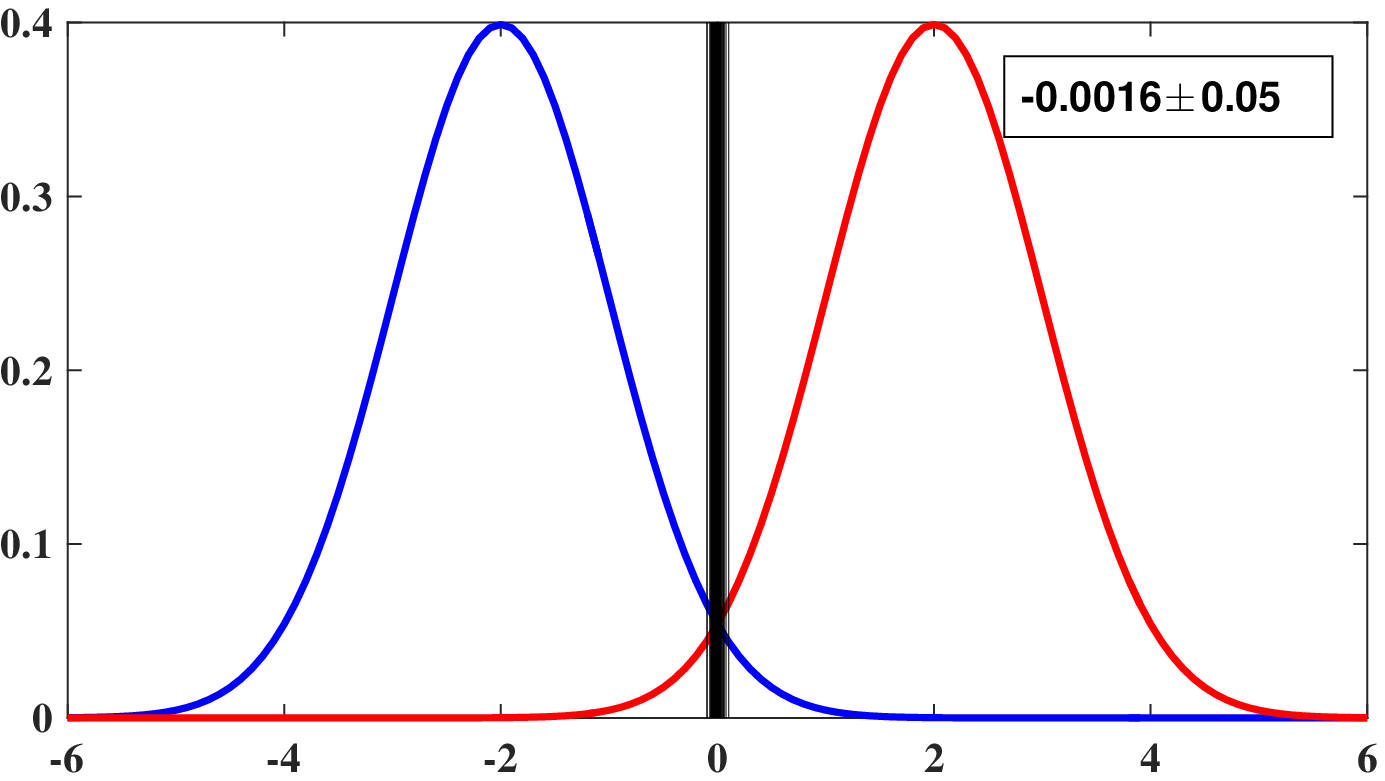}}
\caption{Results of PEGASOS and SGTSVM on 100 artificial datasets, where the 100 upright black solid lines are final classifiers.}\label{Art1}
\end{center}
\end{figure*}

Secondly, we test the stability of SGTSVM compared with PEGASOS. $100$ datasets were generated randomly, and each dataset contain $10,000$ samples in $R$, where $5,000$ negative samples are from normal distribution $N(-2,1)$ and $5,000$ positive ones are from $N(2,1)$. The best classification point is at zero. We implemented PEGASOS and SGTSVM without any restrictions on the $100$ datasets and obtained $100$ classifiers shown in Figure \ref{Art1}, where the upper right digit is the mean of these lines together with their standard deviation (the parameters $c$ in PEGASOS, $c_1$, $c_2$, $c_3$, and $c_4$ in SGTSVM were fixed to $0.1$). It is clear that our SGTSVM obtains much more compact classification lines than PEGASOS. The mean line of SGTSVM is at $-0.0016$ which is closer to zero and its standard deviation is smaller than PEGASOS. In order to investigate the effect of sampling, PEGASOS and SGTSVM were implemented on the above $100$ datasets with the restricted sampling (i.e., some possible support vectors from negative samples in SVM and the samples close to these support vectors are invisible for sampling). Figure \ref{Art1_DelSV} shows the results of PEGASOS and SGTSVM, where the dash line denotes that the samples in this scope are invisible for sampling. From Figure \ref{Art1_DelSV}, it can be seen that the classification lines by PEGASOS fall into two regions, while SGTSVM obtains a compact region. Thus, it means that the possible support vectors significantly influence PEGASOS, while SGTSVM relatively relies on the data distribution. From Figures \ref{Art1} and \ref{Art1_DelSV}, PEGASOS always acquires a mean classification line further from zero with a larger standard deviation than SGTSVM. Therefore, SGTSVM is more stable than PEGASOS on these datasets with or without restricted sampling. To further show the classifiers' stability, we recorded the classification accuracies ($\%$) of PEGASOS and SGTSVM on one of the $100$ datasets. PEGASOS and SGTSVM were implemented $100$ times on this dataset, where the parameters were set as before and two methods were iterated $200$ times. Every accuracies of these methods are reported in Figure \ref{Art1_AC}. From Figure \ref{Art1_AC}, the accuracies of SGTSVM belong to $[99.0,99.5]$ while PEGASOS is $[96.5,99.5]$, which indicates SGTSVM is more stable than PEGASOS from the aspect of classification result. Although PEGASOS obtains the highest accuracy in this test, SGTSVM obtains higher accuracies than PEGASOS in most cases.

\begin{figure*}
\begin{center}
\subfigure[PEGASOS]
{\includegraphics[width=0.35\textheight]{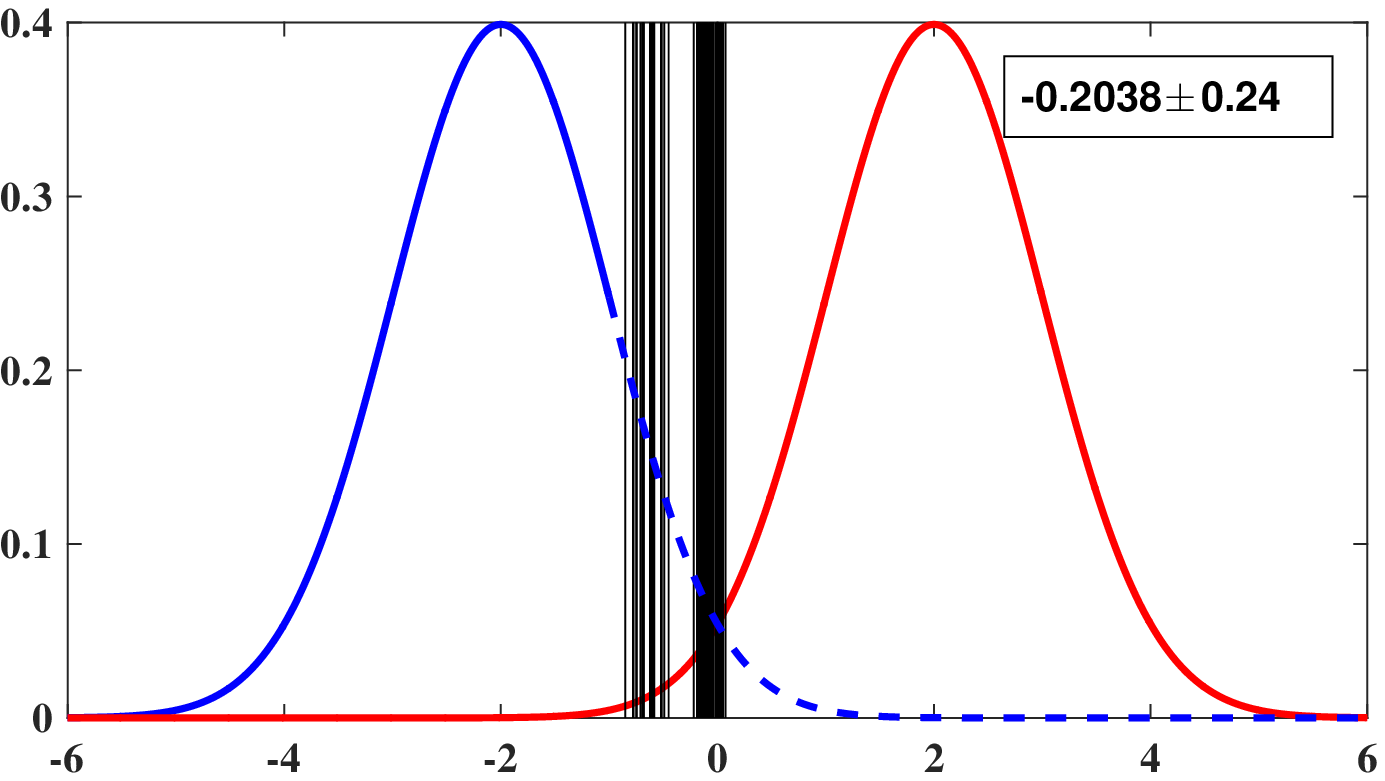}}
\subfigure[SGTSVM]
{\includegraphics[width=0.35\textheight]{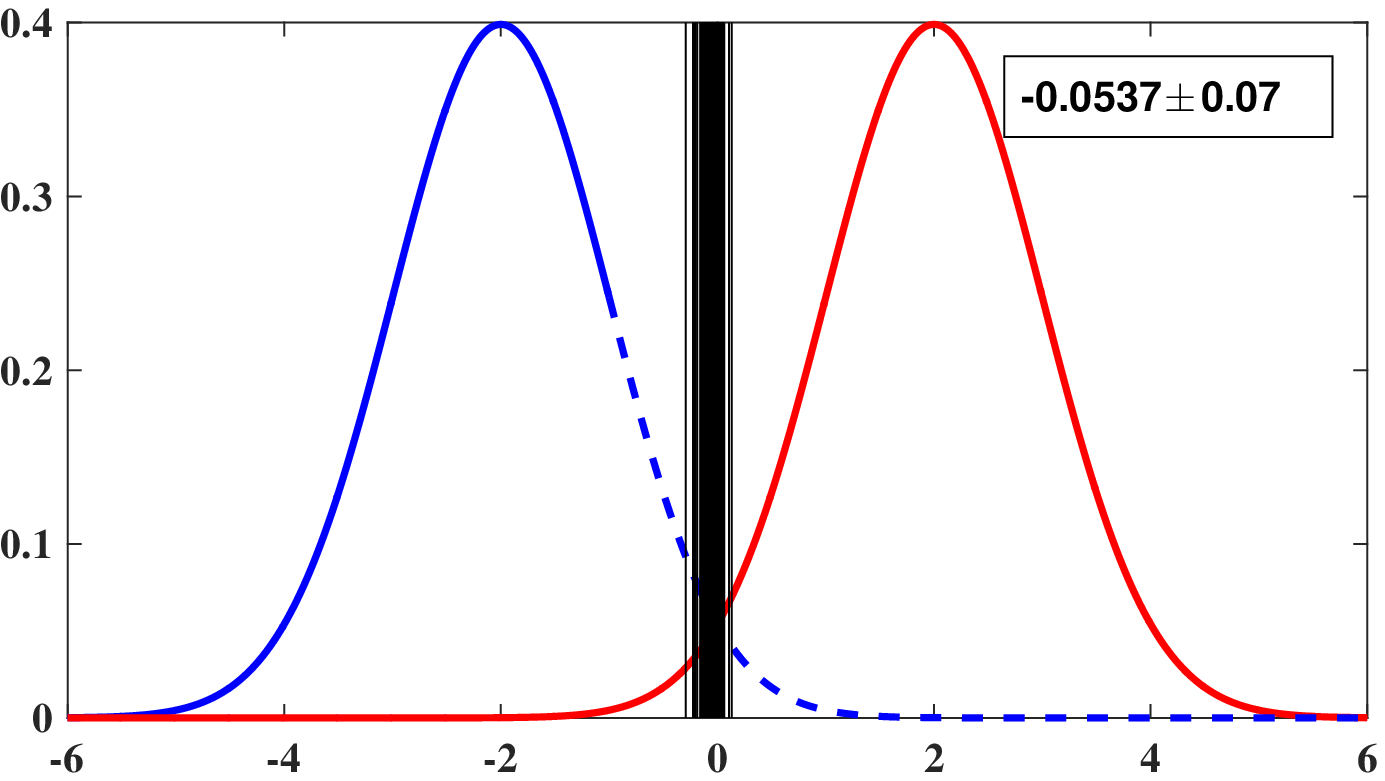}}
\caption{Results of PEGASOS and SGTSVM on 100 artificial datasets, where the 100 upright black solid lines are final classifiers, and the samples on the dash line is invisible for sampling.}\label{Art1_DelSV}
\end{center}
\end{figure*}

\begin{figure*}
\begin{center}
{\includegraphics[width=0.5\textheight]{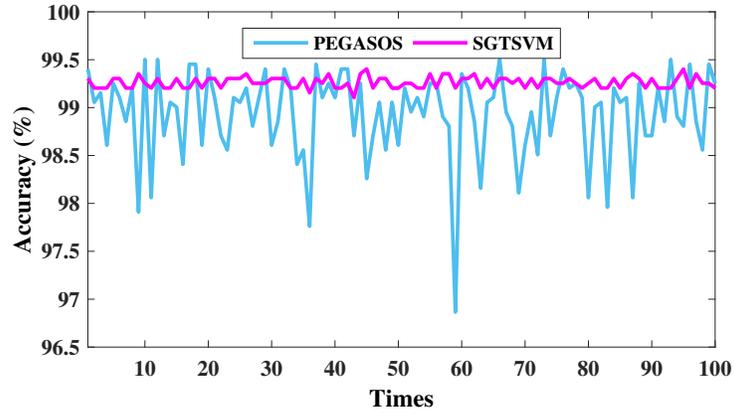}}
\caption{Accuracies of PEGASOS and SGTSVM on a normal distribution dataset, where each method is implemented 100 times.}\label{Art1_AC}
\end{center}
\end{figure*}

\begin{figure*}
\begin{center}
{\includegraphics[width=0.6\textheight]{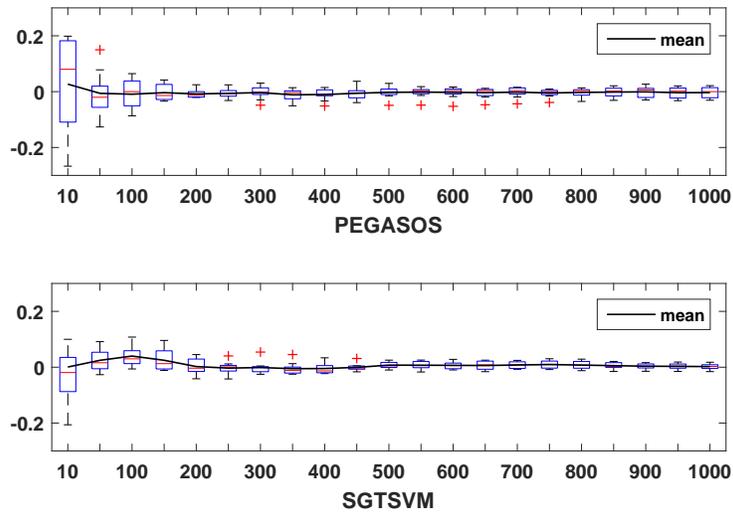}}
\caption{Results of PEGASOS and SGTSVM on a normal distribution dataset, where each method is implemented 10 times. The horizontal axis is the iteration and the vertical one is the classification location.}\label{Art2_box}
\end{center}
\end{figure*}
\begin{figure*}
\begin{center}
{\includegraphics[width=0.35\textheight]{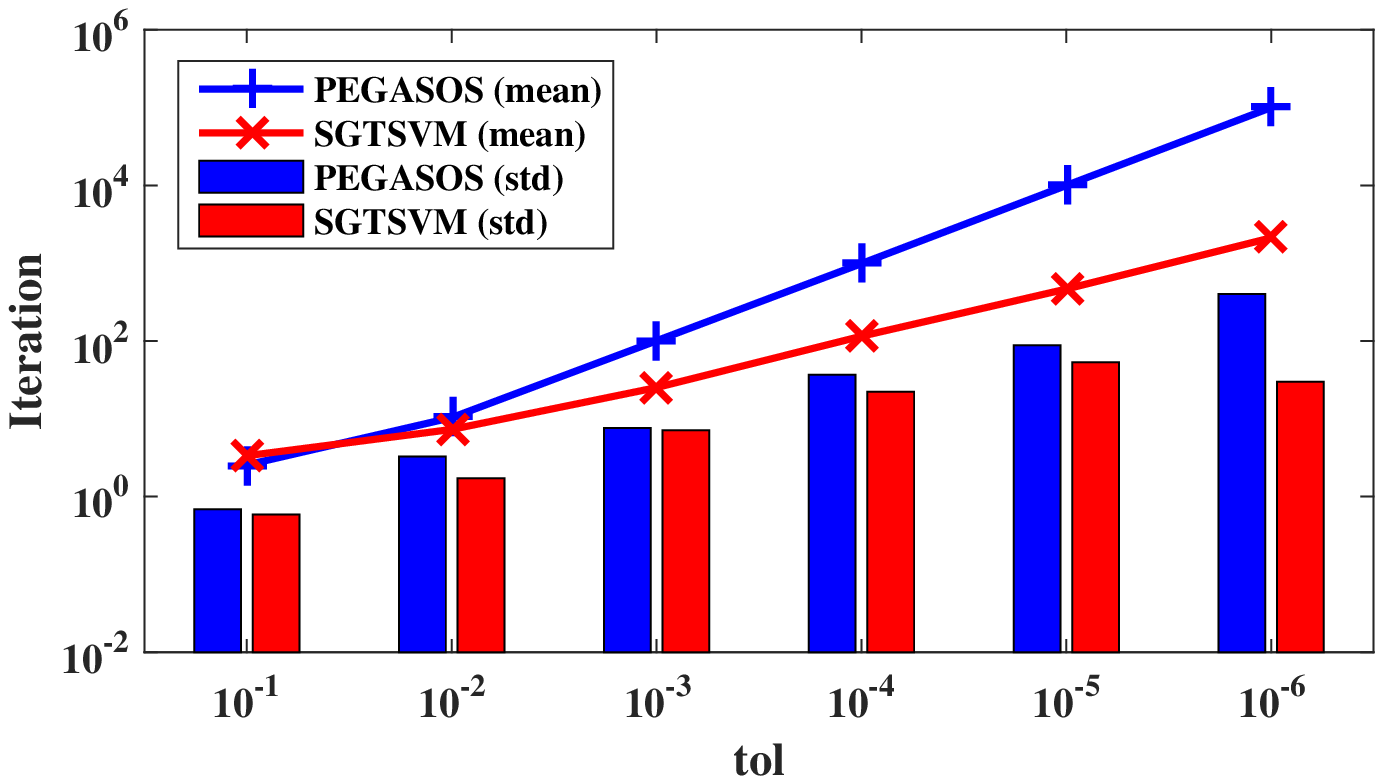}}
{\includegraphics[width=0.35\textheight]{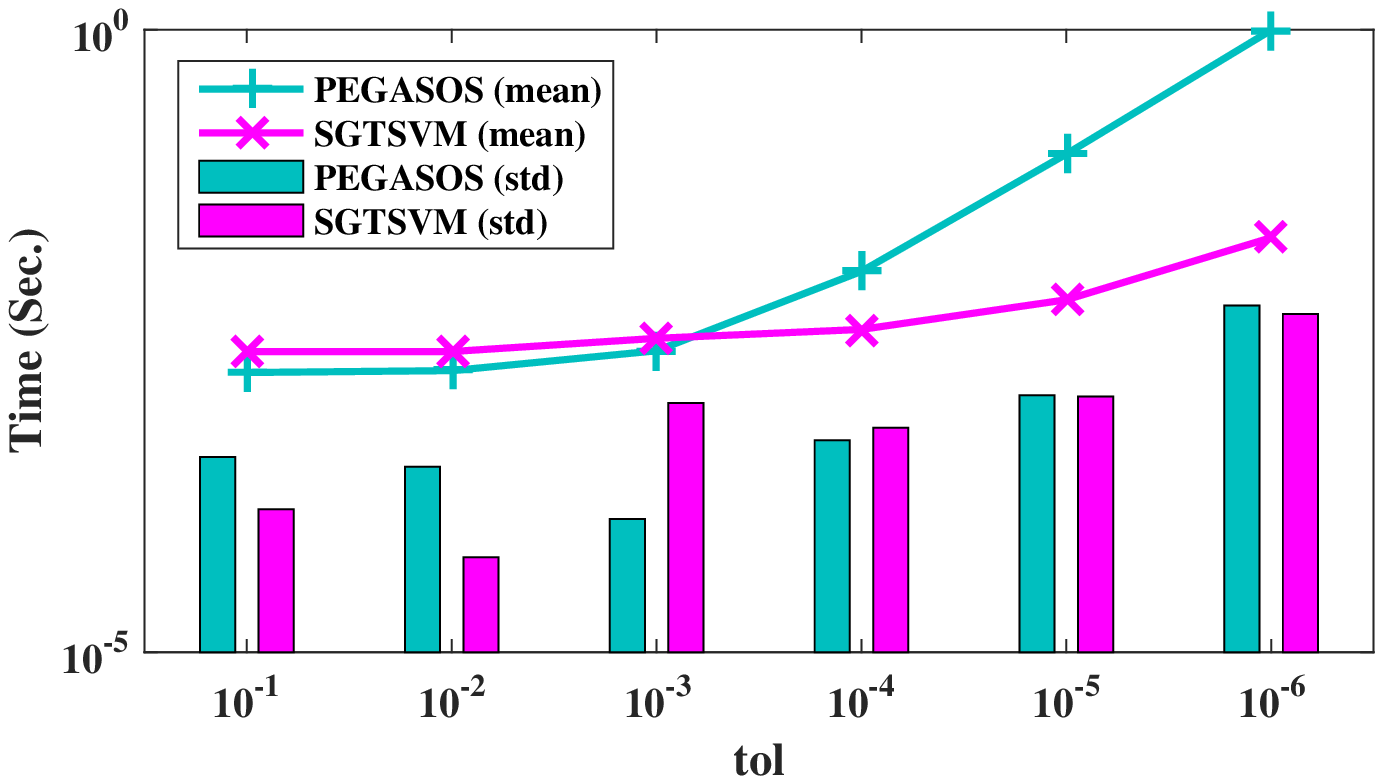}}
\caption{Iteration and time of PEGASOS and SGTSVM on a normal distribution dataset, where each method is implemented 100 times.}\label{Art2}
\end{center}
\end{figure*}

Finally, we test the convergence of PEGASOS and SGTSVM. A dataset contains $20,000$ samples in $R$ was generated randomly, where $10,000$ negative samples are from normal distribution $N(-2,1)$ and $10,000$ positive ones are from $N(2,1)$. PEGASOS and SGTSVM were implemented $10$ times and each method was iterated $1,000$ times. The current classification locations for different iterations were reported in Figure \ref{Art2_box}, where the horizontal axis is the iteration and the vertical one is the classification location.
From Figure \ref{Art2_box}, it can be seen that: (i) the initial selected samples do not very affect both PEGASOS and SGTSVM after iterating $150$ times; (ii) after iterating $100$ times, the classification locations of two methods are centralized to zero and the error is less than $0.1$; (iii) it is important that PEGASOS gets higher error after iterating $800$ times than SGTSVM, indicates PEGASOS converges slower than SGTSVM. To more precisely discuss the convergence, PEGASOS and SGTSVM were implemented $100$ times and each method was terminated by the solution error parameter $tol$ (more details about $tol$ can be found in Algorithm 1). $tol$ is selected from $\{10^i|i=-1,-2,\ldots,-6\}$, and the corresponding iteration and spent time are reported in Figure \ref{Art2}. It is clear from Figure \ref{Art2} that our SGTSVM converges faster than PEGASOS when $tol\leq10^{-3}$. Moreover, if one needs smaller solution error such as $tol=10^{-4}$ or $tol=10^{-5}$, the iterations of PEGASOS would be about $10$ times more than SGTSVM, and it would be $100$ times when $tol=10^{-6}$ (thus the learning time between PEGASOS and SGTSVM is more than a hundredfold). Therefore, SGTSVM converges much faster than PEGASOS.

\begin{table}
\caption{The details of the large scale datasets.} \centering
\begin{tabular}{lllll}
\hline\noalign{\smallskip}
Data&Name & No. of samples &  Dimension & Ratio \\
\hline
(a)&  Skin   &~~~~245,057   & ~3   & 0.262   \\
\hline
(b)& Gashome&~~~~928,990&10&0.578\\
\hline
(c)&Susy&~~5,000,000&18 &0.844\\
\hline
(d)&Kddcup&~~4,898,432&41&0.248\\
\hline
(e)&Gas& ~~8,386,764&16&0.077\\
\hline
(f)&Hepmass&10,500,000&28&1.000\\
\hline \noalign{\smallskip}
\end{tabular} \label{Dataset}\\
\end{table}

\subsection{Large scale datasets}
To test the feasibility of these methods on large scale datasets, we ran SVM, PEGASOS, and SGTSVM on six large scale datasets \cite{UCI}. Table \ref{Dataset} shows the details of the large scale datasets, where Ratio in Table \ref{Dataset} is the sample number of positive class than negative one. Each dataset is split into two subsets where one (including $90\%$ samples) is used for training and the other (including $10\%$ samples) is used for testing. SVM is implemented by Liblinear \cite{LIBLINEAR}, while PEGASOS and SGTSVM are implemented by the softwares written in C language. The corresponding softwares can be downloaded from \url{http://www.optimal-group.org/Resource/SGTSVM.html}. For nonlinear SGTSVM, the reduced kernel \cite{RSVM} is used and the kernel size is fixed to $100$.

\begin{figure*}
\begin{center}
\subfigure[PEGASOS]
{\includegraphics[width=0.35\textheight]{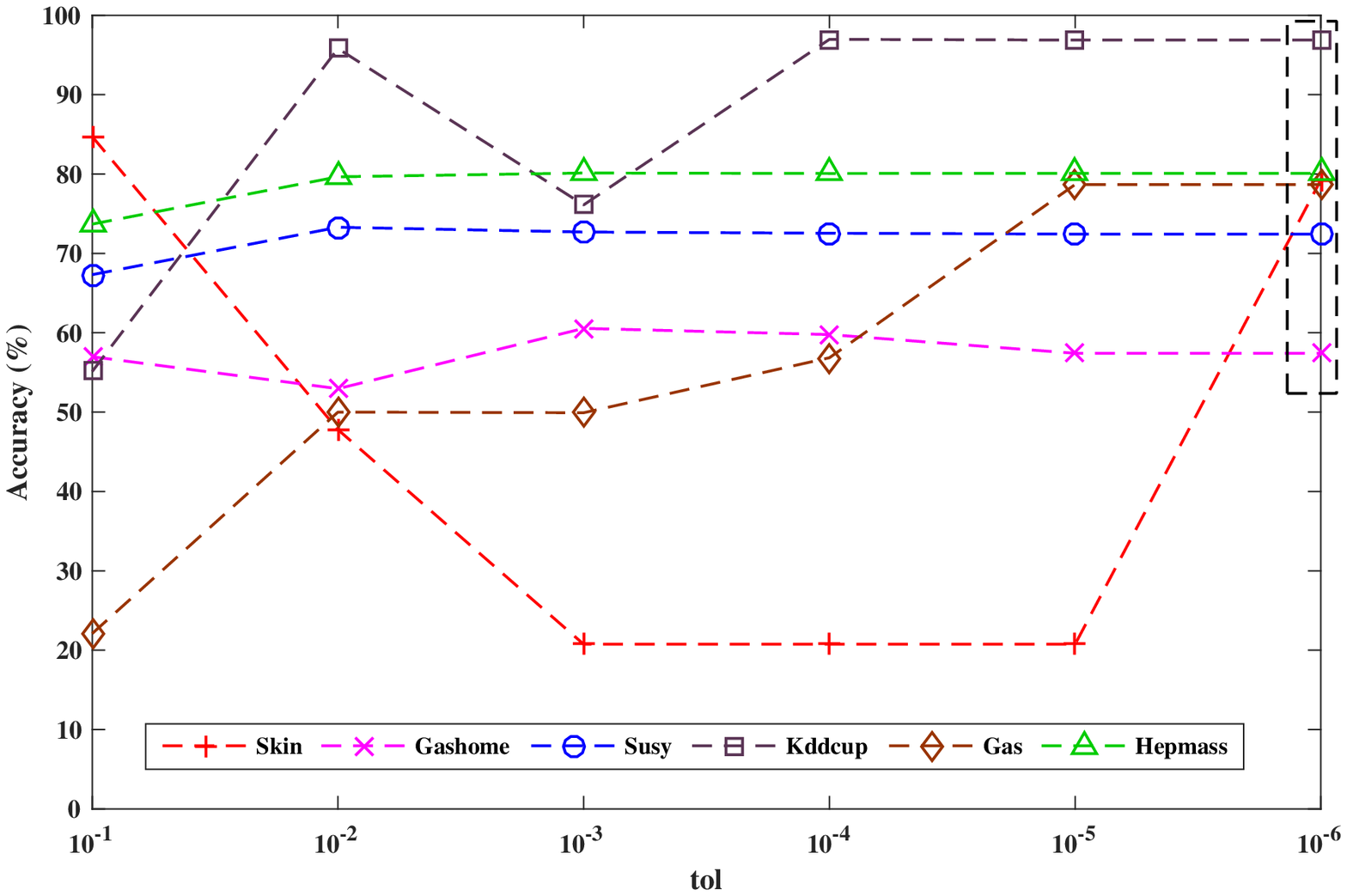}}
\subfigure[PEGASOS]
{\includegraphics[width=0.35\textheight]{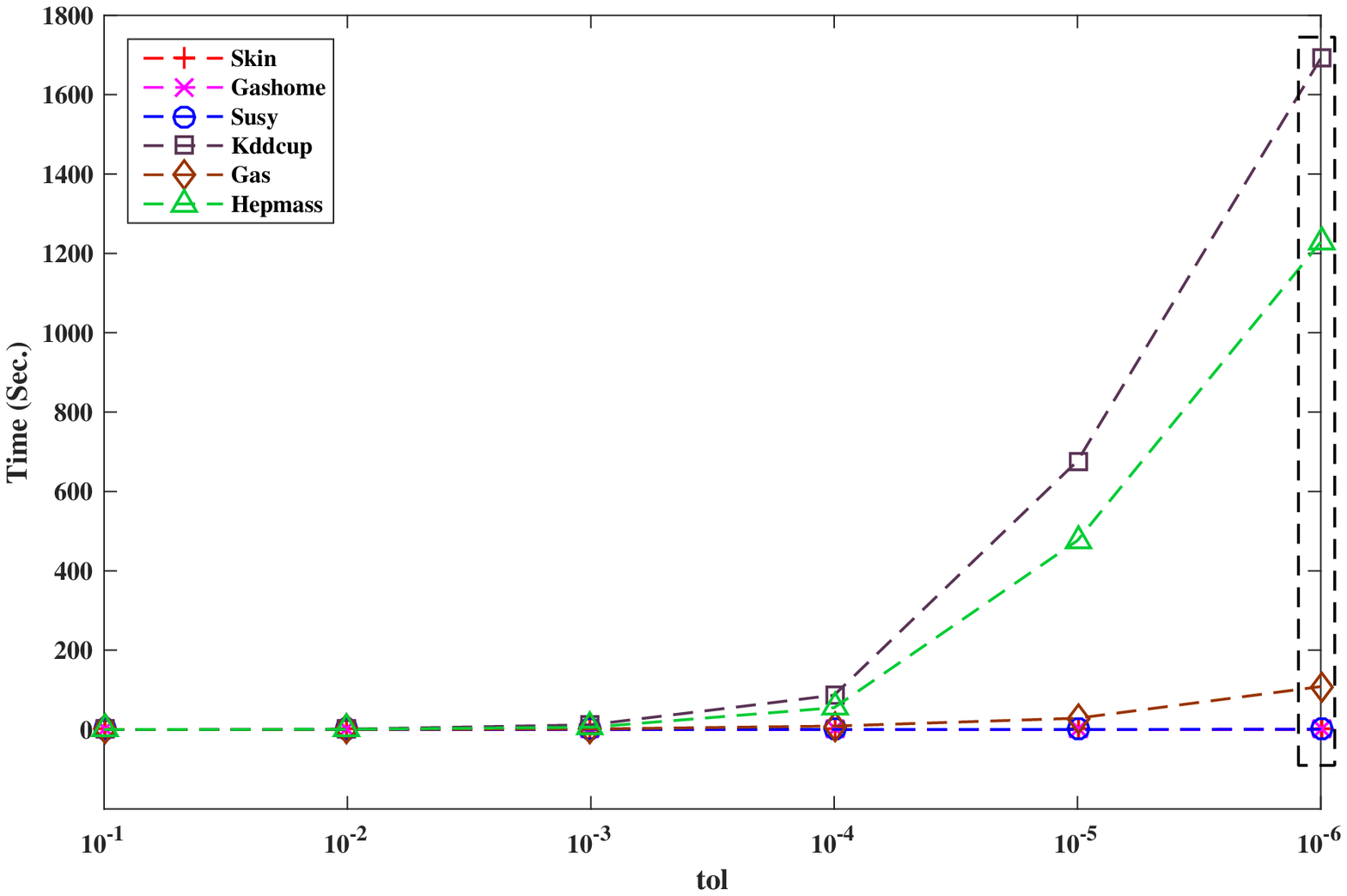}}
\subfigure[SGTSVM$^\dag$]
{\includegraphics[width=0.350\textheight]{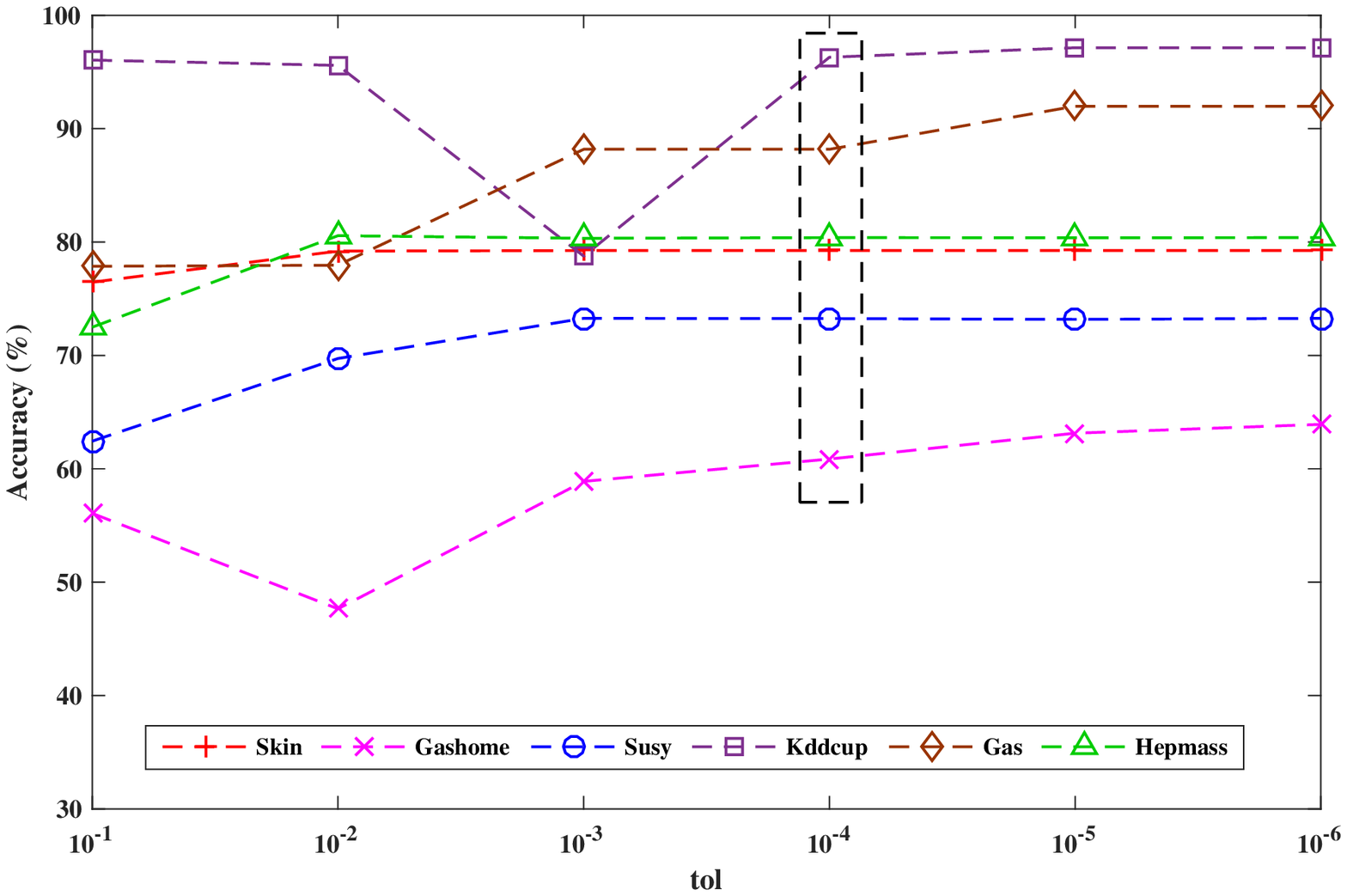}}
\subfigure[SGTSVM$^\dag$]
{\includegraphics[width=0.350\textheight]{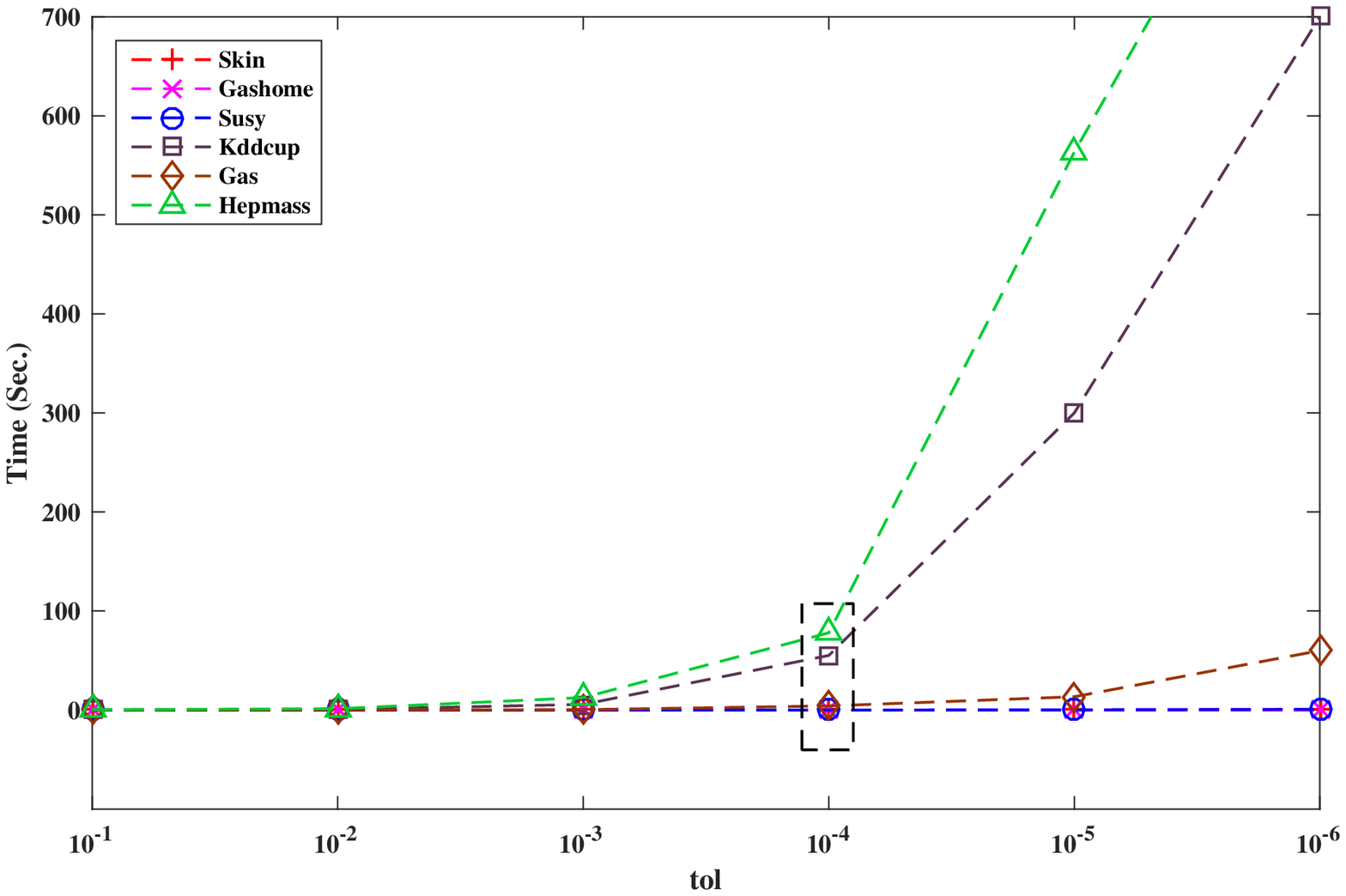}}
\subfigure[SGTSVM$^\sharp$]
{\includegraphics[width=0.350\textheight]{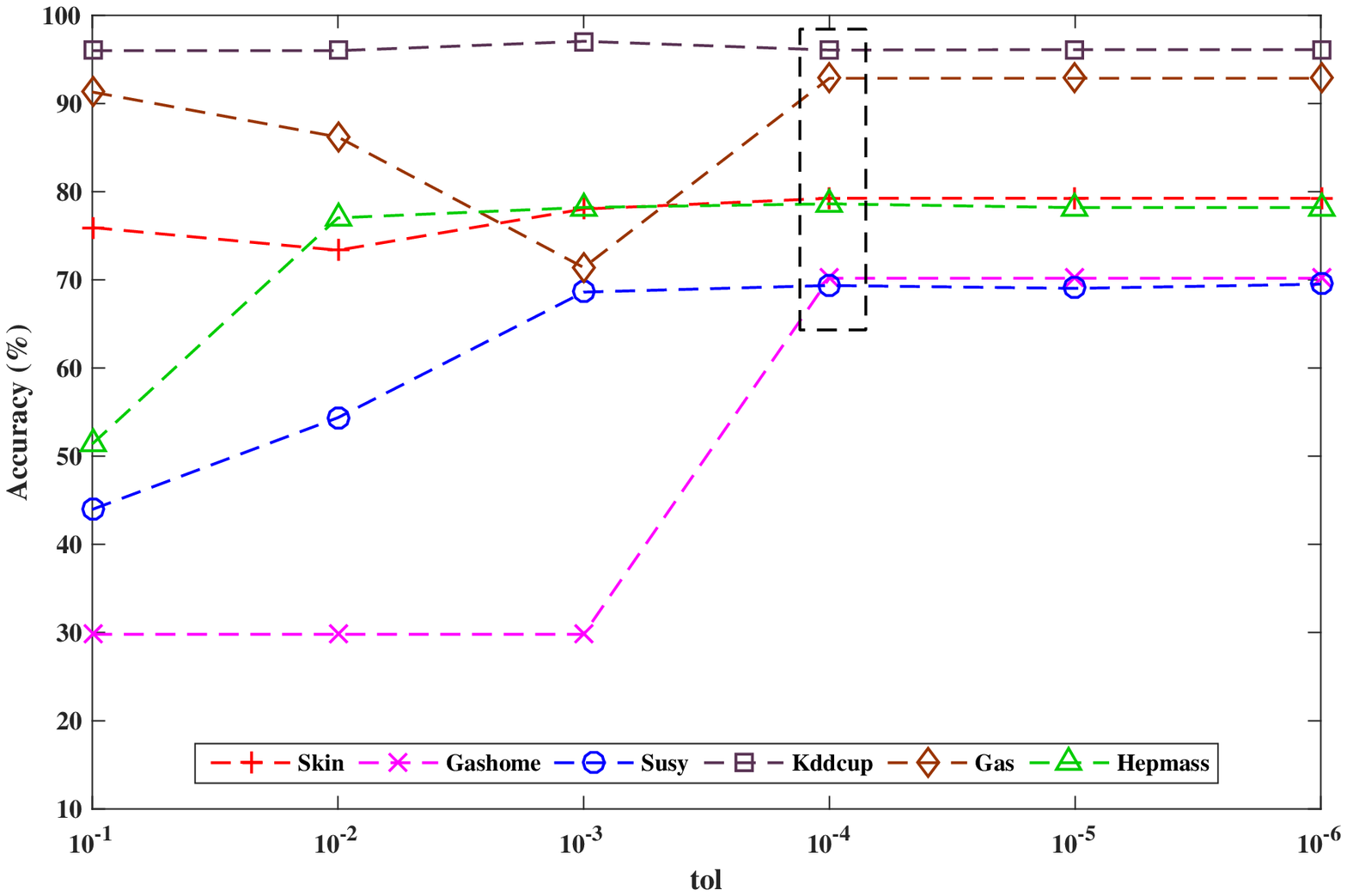}}
\subfigure[SGTSVM$^\sharp$]
{\includegraphics[width=0.350\textheight]{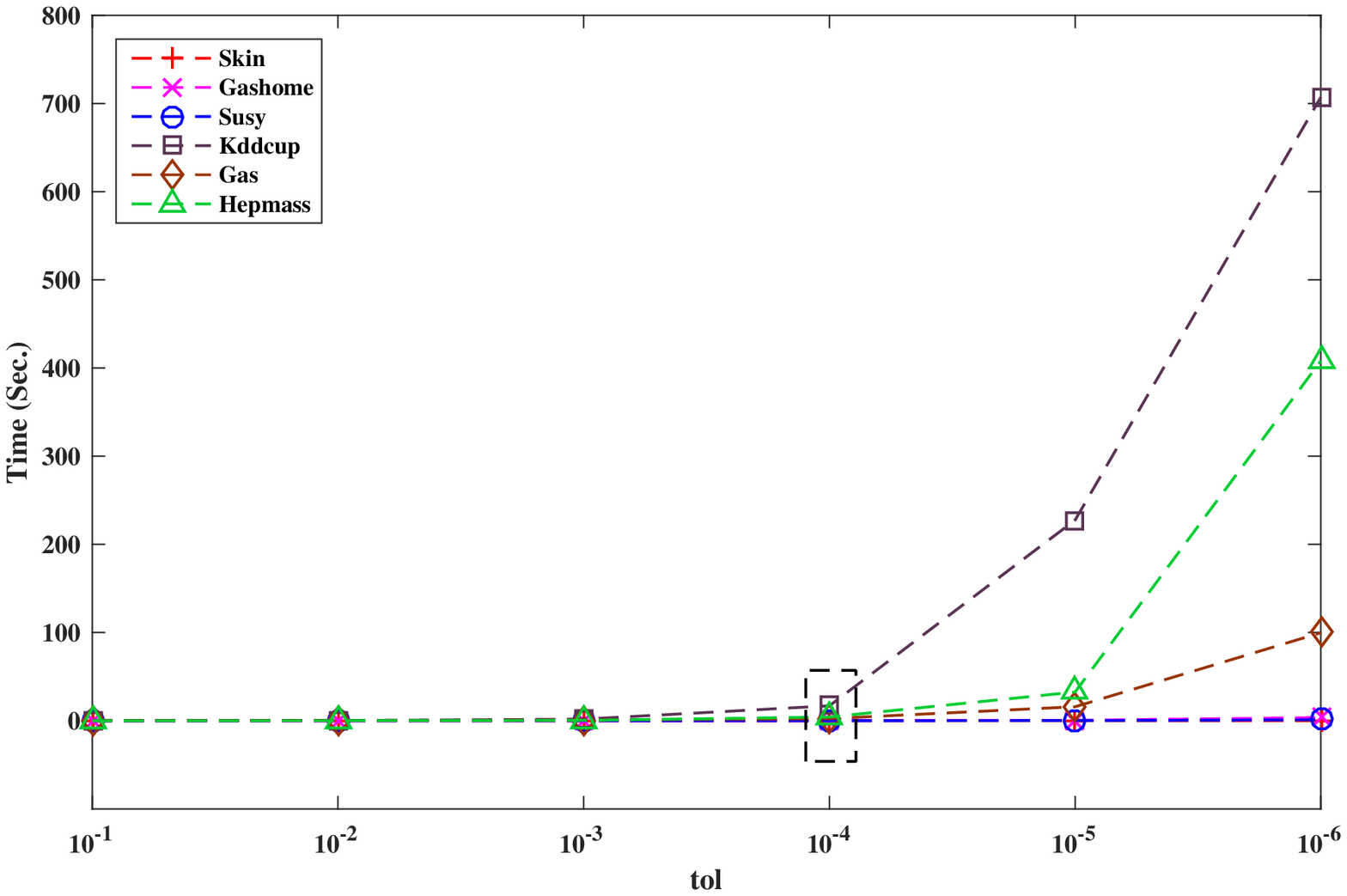}}
\\
\caption{The accuracy and learning time of PEGASOS, linear SGTSVM ($^\dag$), and nonlinear SGTSVM ($^\sharp$) on six large scale datasets. The dashed box corresponds to the chosen parameter $tol$.} \label{TolAC}
\end{center}
\end{figure*}

First, let us test the influence of parameter $tol$ on PEGASOS and SGTSVM. These methods were implemented on the large scale datasets, where $tol$ was respectively set to $\{10^i|i=-1,-2,\ldots,-6\}$ and other parameters were fixed to $0.1$. The testing accuracy and learning time are reported in Figure \ref{TolAC}. By comparing Figure \ref{TolAC} (a), (c), and (e), it can bee seen that our SGTSVM (including linear and nonlinear cases) is more stable than PEGASOS when $tol\leq10^{-4}$. In order to select a high accuracy with an acceptable learning time from Figure \ref{TolAC}, $tol$ is set to $10^{-6}$ for PEGASOS, and it is set to $10^{-4}$ for SGTSVM.

\begin{table}
\caption{The results on the large scale datasets.} \centering
\begin{tabular}{llllll}
\hline\noalign{\smallskip}
Data & & SVM &  PEGASOS & SGTSVM$^\dag$ &SGTSVM$^\sharp$ \\
\noalign{\smallskip}\hline\noalign{\smallskip}
Skin &validation(\%) &78.87 &82.46 &$\mathbf{85.23}$&84.70\\
245,057$\times$3 &testing(\%) &84.28 &85.39 &$\mathbf{87.70}$&85.34  \\
\hline
Gashome &validation(\%) &49.11 &70.09 &67.50&$\mathbf{74.49}$\\
919,438$\times$10&testing(\%)&82.57&72.85&76.09&$\mathbf{89.13}$\\
\hline
Susy&validation(\%)&$\mathbf{78.41}$&54.11&76.14&69.90\\
5,000,000$\times$18&testing(\%)&$\mathbf{78.52}$&56.44&75.09&68.61\\
\hline
Kddcup&validation(\%)&*&$\mathbf{96.39}$&95.24&93.19\\
4,898,432$\times$41&testing(\%)&*&96.42&97.45&$\mathbf{99.20}$\\
\hline
Gas&validation(\%)&*&69.77&89.73&$\mathbf{92.60}$\\
8,386,764$\times$16&testing(\%)&*&50.54&92.45&$\mathbf{92.86}$\\
\hline
Hepmass&validation(\%)&*&80.63&80.80&$\mathbf{82.18}$\\
10,500,000$\times$28&testing(\%)&*&80.84&$\mathbf{81.10}$&79.59\\
\hline \noalign{\smallskip}
\end{tabular} \label{LargeAC}
$^\dag linear ~case$;$^\sharp nonlinear ~case$; $^*out ~of ~memory$.
\end{table}

\begin{table}
\caption{The optimal parameters of SVM, PEGASOS, and SGTSVM.} \centering
\begin{tabular}{llllll}
\hline\noalign{\smallskip}
Data & & SVM &  PEGASOS & SGTSVM$^\dag$ &SGTSVM$^\sharp$ \\
& &c & c &$c_1=c_3,c_2=c_4$&$c_1=c_3,c_2=c_4,\mu$\\
&&$2^i$&$2^i$&$2^i,2^j$&$2^i,2^j,2^k$\\
\noalign{\smallskip}\hline\noalign{\smallskip}
Skin  &validation&-1 &-6 &0,-5&-6,-5,-3\\
&testing&-1&-4&1,-6&-1,0,-9\\
\hline
Gashome& validation& 0&-6&-4,-5&-3,-5,-2\\
&testing& -1&-1&-8,-7&-8,-1,-2\\
\hline
Susy& validation&1&0&-2,-6&-3,-1,-4\\
&testing&0&-7&-1,-3&-3,-3,-3\\
\hline
Kddcup&validation& NA& -6&-8,-4&0,-3,-4\\
&testing&NA&-2&-8,-4&-6,-1,-8\\
\hline
Gas&validation&NA&-1&-4,0&-1,-1,-6\\
&testing&NA&1&-3,1&-4,-8,-6\\
\hline
Hepmass&validation&NA&0&-1,-2&-4,-1,-3\\
&testing&NA&0&0,-2&-4,-2,-3\\
\hline \noalign{\smallskip}
\end{tabular} \label{Param}
$^\dag linear ~case$;$^\sharp nonlinear ~case$.
\end{table}

\begin{figure*}
\begin{center}
{\includegraphics[width=0.6\textheight]{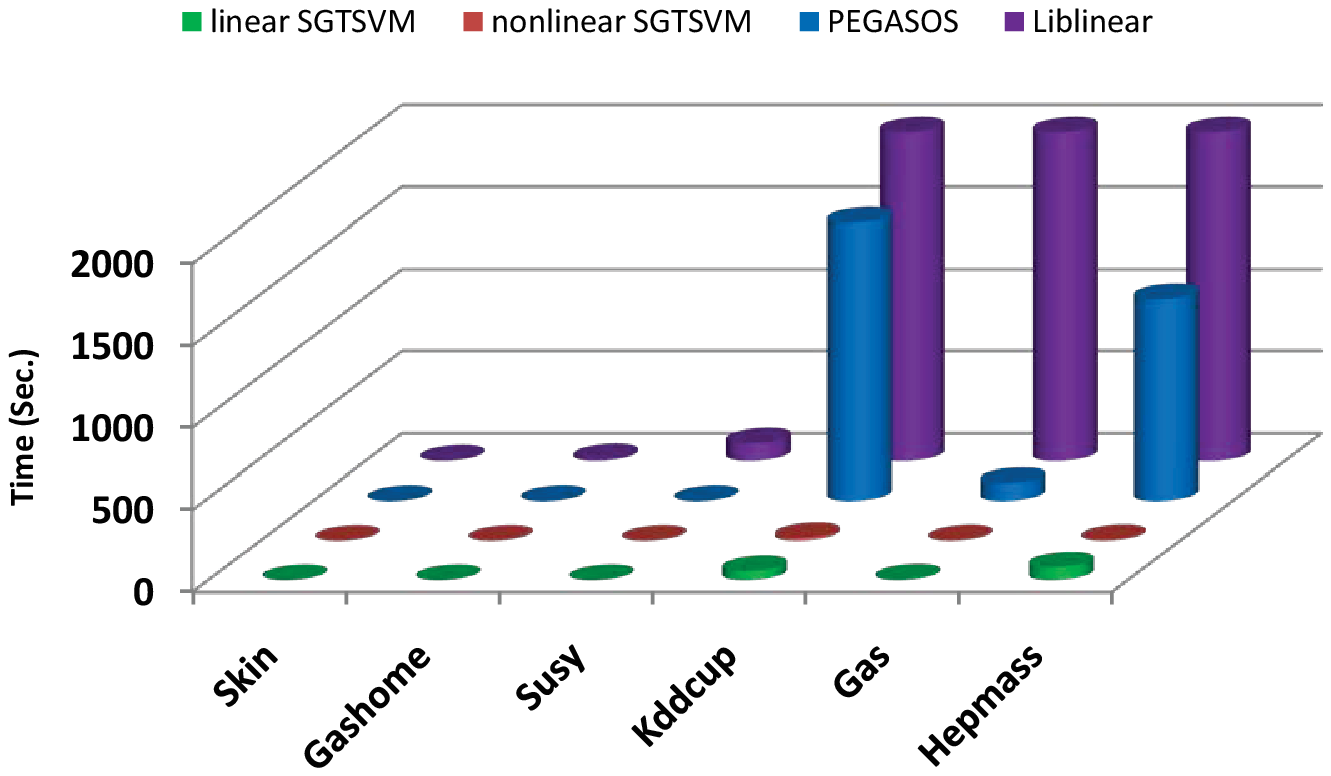}}
\caption{Learning time of SGTSVM, PEGASOS, and Liblinear on the large scale datasets with the optimal parameters.}\label{LargeTime}
\end{center}
\end{figure*}

Then, we compare SVM and PEGASOS with our SGTSVM with fixed $tol$ on these datasets. These methods' accuracies are recorded in Table \ref{LargeAC}, where validation accuracy is obtained by 5-fold cross validation on the training subset, and testing accuracy is obtained by the testing subset. The parameters $c$ in SVM and PEGASOS, $c_1$, $c_2$, $c_3$, and $c_4$ in SGTSVM are selected from $\{2^i|i=-8,-7,\ldots,1\}$, and the Gaussian kernel parameter $\mu$ in nonlinear SGTSVM is selected from $\{2^i|i=-10,-9,\ldots,-1\}$. For simplicity, we also set $c_1=c_3$ and $c_2=c_4$ in SGTSVM. The optimal parameters are recorded in Table \ref{Param}. From Table \ref{LargeAC}, it is obvious that our SGTSVM owns the highest accuracies on $9$ groups of comparisons, and performs as well as SVM or PEGASOS on the other $3$ groups. However, SVM performs much worse than SGTSVM on the dataset Gashome and cannot work on three much larger datasets. Though PEGASOS can work on these datasets, it performs much worse than SGTSVM on Susy and Gas. To further comparing the learning time of these methods, we report the one-run time in Figure \ref{LargeTime} with the optimal parameters. It is obvious that SGTSVM (including linear and nonlinear cases) is much faster than the others. Thus, our SGTSVM is comparable to SVM and PEGASOS on these large scale datasets. In addition, the softwares of SGTSVM and PEGASOS need much less RAM than Liblinear (the software of SVM). In detail, Liblinear needs store the entire training set in RAM, while PEGASOS and SGTSVM only store a subset related to the iteration. Due to the required memory of Liblinear increases with the size of dataset, it tends to out of memory with the increasing data size, while the same thing does not appear in PEGASOS or SGTSVM.

\section{Conclusion}
The stochastic gradient twin support vector machines (SGTSVM) based on stochastic gradient decent algorithm has been
proposed. By hiring the nonparallel hyperplanes, SGTSVM is more stable on stochastic sampling than PEGASOS. In theory, we prove that SGTSVM is convergent, and it is an approximation of TWSVM with uniform sampling. Experimental results
have confirmed the merits of SGTSVM and shown our SGTSVM has better accuracy compared with
Liblinear and PEGASOS with the fastest learning speed. For practical convenience,
the corresponding SGTSVM codes (including Matlab and C language) can be downloaded from
\url{http://www.optimal-group.org/Resource/SGTSVM.html}. For the future work, it is possible to design some special sampling for SGTSVM to obtain more powerful performance, together with applying SGTSVM on the bigdata problems.



\section*{Acknowledgment}
This work is supported by the National Natural Science Foundation of
China (Nos. 11501310, 11201426, and 11371365), the Natural Science Foundation of Inner Mongolia Autonomous Region of China (No. 2015BS0606), and the Zhejiang Provincial Natural Science Foundation of China (No. LY15F\\030013).

\bibliographystyle{plain}
\bibliography{FBib}

\end{document}